\documentclass{article} 
\usepackage{PRIMEarxiv}

\usepackage{amsmath,amsfonts,bm}

\def\eqref#1{equation~\ref{#1}}

\def\ceil#1{\lceil #1 \rceil}

\def\1{\bm{1}}

\DeclareMathAlphabet{\mathsfit}{\encodingdefault}{\sfdefault}{m}{sl}
\SetMathAlphabet{\mathsfit}{bold}{\encodingdefault}{\sfdefault}{bx}{n}

\usepackage{hyperref}
\usepackage{url}

\usepackage{enumitem,kantlipsum}
\usepackage[T1]{fontenc}    
\usepackage{hyperref}       
\usepackage{url}            
\usepackage{booktabs}       
\usepackage{amsfonts}       
\usepackage{nicefrac}       
\usepackage{microtype}      
\usepackage{floatrow}
\usepackage{xcolor}
\let\oldnl\nl
\newcommand{\nonl}{\renewcommand{\nl}{\let\nl\oldnl}}

\usepackage{adjustbox}
\usepackage{blindtext}
\usepackage{setspace}
\usepackage[ruled, lined, linesnumbered, commentsnumbered, longend]{algorithm2e}

\SetCommentSty{mycommfont}
\def\guyblue#1{{\color{black}{#1}}}
\def\guyPr {\textbf{Pr}}
\usepackage{anyfontsize}
\usepackage{dblfloatfix}   
\usepackage{enumitem,kantlipsum}

\usepackage{floatrow}
\usepackage{tabularx}
\usepackage{amsfonts}
\usepackage{amssymb}
\usepackage{amsmath}
\usepackage{caption}
\usepackage{subcaption}
\usepackage{dsfont}
\usepackage{mathtools}
\usepackage{amsthm}
\usepackage{multirow}
\usepackage[normalem]{ulem}
\useunder{\uline}{\ul}{}
\usepackage{multirow}

\newcommand{\vect}[1]{\boldsymbol{#1}}
\newfloatcommand{capbtabbox}{table}[][\FBwidth]
\newtheorem{Theorem}{Theorem}[section]
\newtheorem{lemma}[Theorem]{Lemma}
\newcommand {\cK} {\mathcal {K} }

\usepackage{microtype}
\usepackage{graphicx}
\usepackage{booktabs} 

\title{Window-Based Distribution Shift Detection for Deep Neural Networks}

\author{Guy Bar-Shalom \\
Department of Computer Science\\
Technion \\
guy.b@cs.technion.ac.il \\
\And
Yonatan Geifman \\
Deci.AI \\
yonatan.g@cs.technion.ac.il \\
\And
Ran El-Yaniv \\
Department of Computer Science\\
Technion, Deci.AI \\
rani@cs.technion.ac.il \\
}
\begin{document}

\maketitle

\begin{abstract}
To deploy and operate deep neural models in production, the quality of their predictions, which might be contaminated benignly or manipulated maliciously by input distributional deviations, must be monitored and assessed. Specifically, we study the case of monitoring the healthy operation of a deep neural network (DNN) receiving a stream of data, with the aim of detecting input distributional deviations over which the quality of the network's predictions is potentially damaged. Using selective prediction principles, we propose a distribution deviation detection method for DNNs. The proposed method is derived from a tight coverage generalization bound computed over a sample of instances drawn from the true underlying distribution. Based on this bound, our detector continuously monitors the operation of the network out-of-sample over a test window and fires off an alarm whenever a deviation is detected. Our novel detection method performs on-par or better than the state-of-the-art, while consuming substantially lower computation time (five orders of magnitude reduction) and space complexities. Unlike previous methods, which require at least linear dependence on the size of the source distribution for each detection, rendering them inapplicable to ``Google-Scale'' datasets, our approach eliminates this dependence, making it suitable for real-world applications. 
\end{abstract}

\section{Introduction}
\label{sec: intro}
A wide range of artificial intelligence applications and services rely on deep neural models because of their remarkable accuracy. When a trained model is deployed in production, its operation should be monitored for abnormal behavior, and a flag should be raised if such is detected. Corrective measures can be taken if the underlying cause of the abnormal behavior is identified. For example, simple distributional changes may only require retraining with fresh data, while more severe cases may require redesigning the model (e.g., when new classes emerge).

We are concerned with distribution shift detection in the context of deep neural models and consider the following setting. Pretrained model $f$ is given, and we presume it was trained with data sampled from some distribution $P$. In addition to the dataset used in training $f$, we are also given an additional sample of data from $P$, which is used to train a detector $D$ (we refer to this as the detection-training set or source set). While $f$ is used in production to process a stream of emerging input data, we continually feed $D$ with the most recent window $W_k$ of $k$ input elements. The detector also has access to the final layers of the model $f$ and should be able to determine whether the data contained in $W_k$ came from a distribution different from $P$.
We emphasize that in this paper we are not considering the problem of identifying \emph{single-instance} out-of-distribution or outlier instances 
\cite{liang2018enhancing,DBLP:conf/iclr/HendrycksG17,hendrycks2018deep,DBLP:conf/nips/GolanE18, ren2019likelihood, nalisnick2019detecting, DBLP:journals/corr/abs-2106-04015, fort2021exploring}, but rather the information residing in a population of $k$ instances. \guyblue{While it may seem straightforward to apply single-instance detectors to a window (by applying the detector to each instance in the window), this approach can be computationally expensive since such methods are not designed for window-based tasks; see discussion in Section~\ref{sec: related work}.
Moreover, we demonstrate here that \emph{computationally feasible} single-instance methods can fail to detect population-based deviations. We emphasize that we are not concerned in characterizing types of distribution shifts, nor do we tackle at all the complementary topic of out-of-distribution robustness.
}


Distribution shift detection has been scarcely considered in the context of deep neural networks (DNNs), however, it is a fundamental topic in machine learning and statistics. The standard method for tackling it is by performing a dimensionality reduction over both the detection-training (source) and test (target) samples, and then applying a two-sample statistical test over these reduced representations to detect a deviation. This is further discussed in Section~\ref{sec: related work}. In particular, deep models can benefit from the semantic representation created by the model itself, which provides meaningful dimensionality reduction that is readily available at the last layers of the model. Using the embedding layer (or softmax) along with statistical two-sample tests was recently  proposed by \cite{DBLP:conf/icml/LiptonWS18} and \cite{Rabanser2019FailingLA} who termed solutions of this structure black-box shift detection (BBSD). 
Using both the univariate Kolmogorov-Smirnov (KS) test and the maximum mean discrepancy (MMD) method, see details below, 
\cite{Rabanser2019FailingLA} achieves impressive detection results when using MNIST and CIFAR-10 as proxies for the distribution $\mathcal{P}$. As shown here, the KS test is also very effective over ImageNet when a stronger model is used (ResNet50 vs ResNet18). However, BBSD methods have the disadvantage of being computationally intensive (and probably inapplicable to read-world datasets) due to the use of two-sample tests between the detection-training set (which can, and is preferred to be the largest possible) and the window $W_k$; a complexity analysis is provided in Section~\ref{sec: Complexity Anylysis}.




We propose a different approach based on selective prediction \cite{JMLR:v11:el-yaniv10a, DBLP:conf/nips/GeifmanE17}, where a model quantifies its prediction uncertainty and abstains from predicting uncertain instances. First, we develop a method for selective prediction with guaranteed coverage. This method identifies the best abstaining threshold and coverage bound for a given pretrained classifier $f$, such that the resulting empirical coverage will not violate the bound with high probability (when abstention is determined using the threshold). The guaranteed coverage method is of independent interest, and it is analogous to selective prediction with guaranteed risk \cite{DBLP:conf/nips/GeifmanE17}. Because the empirical coverage of such a classifier is highly unlikely to violate the bound if the underlying distribution remains the same, a systematic violation indicates a distribution shift. To be more specific, given a detection-training sample $S_m$, our coverage-based detection algorithm computes \guyblue{a fixed number of} tight generalization coverage bounds, which are then used to detect a distribution shift in a window $W_k$ of test data. \guyblue{The proposed detection algorithm exhibits remarkable computational efficiency due to its ability to operate independently of the size of $S_m$ during detection, which is crucial when considering ``Google-scale'' datasets, such as the JFT-3B dataset. In contrast, the currently available distribution shift detectors rely on a framework that requires significantly higher computational requirements (this framework is illustrated in Figure~\ref{typical pipeline} in Appendix~\ref{app:Shift Detection General Framework}). A run-time comparison of these methods and ours is provided in Figure~\ref{fig: complexity}.}

In a comprehensive empirical study, we compared our coverage-based detection algorithm with the best-performing baselines, including the KS approach of \cite{Rabanser2019FailingLA}. \guyblue{Additionally, we investigated the suitability of single-instance detection methods for identifying distribution shifts. For a fair comparison, all methods used the same (publicly available) underlying models: ResNet50, MobileNetV3-S, and ViT-T. To evaluate the effectiveness of our approach, we employed the ImageNet dataset to simulate the source distribution. We then introduced distribution shifts using a range of methods, starting with simple noise and progressing to more sophisticated techniques such as adversarial examples.} Based on these experiments, it is evident that our coverage-based detection method is overall significantly more powerful than the baselines across a wide range of test window sizes. 

To summarize, the contributions of this paper are: (1) A theoretically justified algorithm (Algorithm~\ref{Alg:CD}), that produces a coverage bound, which is of independent interest, and allows for the creation of selective classifiers with guaranteed coverage.
\guyblue{(2) A theoretically motivated  ``windowed'' detection algorithm (Algorithm~\ref{Alg:Detection}), which detects a distribution shift over a window; this proposed algorithm exhibits remarkable efficiency compared to state-of-the-art methods (five orders of magnitude better than the best method).}
(3) A comprehensive empirical study demonstrating significant improvements relative to existing baselines, and introducing the use of single-instance methods for detecting distribution shifts.
\vspace{-0.3cm}

\section{Problem Formulation}
\label{sec:Problem Formulation}
\vspace{-0.2cm}
We consider the problem of detecting distribution shifts in input streams provided to pretrained deep neural models. 
Let $\mathcal{P} \triangleq P_{\boldmath{X}}$ denote a probability distribution over an input space $\mathcal{X}$, and assume that a model $f$ has been trained on a set of instances drawn from $\mathcal{P}$.
Consider a setting where the model $f$ is deployed and while being used in production its input distribution might change or even be attacked by an adversary. Our goal is to detect such events to allow for appropriate action, e.g., retraining the model with respect to the revised distribution.

Inspired by 
\cite{Rabanser2019FailingLA}, we formulate this problem as follows.
We are given a pretrained model $f$, and a detection-training set, $S_m \sim \mathcal{P}^m$. Then we would like to train a detection model to be able to detect a distribution shift; namely, discriminate between windows containing in-distribution (ID) data, and \emph{alternative-distribution} (AD) data. Thus, given an unlabeled test sample window $W_k \sim Q^k$, where $Q$ is a possibly different distribution, the objective is to determine whether $\mathcal{P} \neq Q$. We also ask what is the smallest test sample size $k$ required to determine that $\mathcal{P} \neq Q$. Since typically the detection-training set $S_m$ can be quite large, we further ask whether it is possible to devise an effective detection procedure whose time complexity is $o(m)$.

\vspace{-0.3cm}
\section{Related Work}
\vspace{-0.2cm}
\label{sec: related work}
Distribution shift detection methods often comprise the following two steps: dimensionality reduction, and a two-sample test between the detection-training sample and test samples.
In most cases, these methods are ``lazy'' in the sense that for each test sample, they make a detection decision based on a computation over the entire detection-training sample. Their performance will be sub-optimal if only a subset of the train sample is used. Figure~\ref{typical pipeline} in
Appendix~\ref{app:Shift Detection General Framework} illustrates this general framework.

The use of dimensionality reduction is optional. It can often improve performance by focusing on a less noisy representation of the data.
Dimensionality reduction techniques include no reduction, \emph{principal components analysis} \cite{wold1987principal}, \emph{sparse random projection} \cite{bingham2001random}, \emph{autoencoders} \cite{rumelhart1985learning, DBLP:conf/nips/PuGHYLSC16}, \emph{domain classifiers}, \cite{Rabanser2019FailingLA} and more. 
In this work we focus on \emph{black box shift detection} (BBSD) methods \cite{DBLP:conf/icml/LiptonWS18}, that rely on deep neural representations of the data generated by a pretrained model.
The representation we extract from the model will typically utilize either the softmax outputs, acronymed BBSD-Softmax, or the embeddings, acronymed BBSD-Embeddings; for simplicity, we may omit the BBSD acronym.
Due to the dimensionality of the final representation, multivariate or multiple univariate two-sample tests can be conducted.

By combining BBSD-Softmax with a Kolmogorov-Smirnov (KS) statistical test \cite{massey1951kolmogorov} and using the Bonferroni correction \cite{bland1995multiple}, 
\cite{Rabanser2019FailingLA} achieved state-of-the-art results in distribution shift detection in the context of image classification (MNIST and CIFAR-10). We acronym their method as BBSD-KS-Softmax (or KS-Softmax). The \emph{univariate} KS test processes individual dimensions separately; its statistic is calculated by computing the largest difference $Z$ of the \emph{cumulative density functions} (CDFs) across all dimensions as follows: $    Z =   \sup_{\vect{z}}\lvert F_{\mathcal{P}}(\vect{z}) - F_Q(\vect{z}) \rvert $,
where $F_Q$ and $F_{\mathcal{P}}$ are the empirical CDFs of the detection-training and test data (which are sampled from $\mathcal{P}$ and $Q$, respectively; see Section~\ref{sec:Problem Formulation}). The Bonferroni correction rejects the null hypothesis when the minimal p-value among all tests is less than $\frac{\alpha}{d}$, where $\alpha$ is the significance level 
and $d$ is the number of dimensions. 
Although less conservative approaches to aggregation exist
\cite{heard2018choosing, loughin2004systematic}, they usually assume some dependencies among the tests. The \emph{maximum mean discrepancy} (MMD) method \cite{DBLP:journals/jmlr/GrettonBRSS12} is a kernel-based multivariate test that can be used to distinguish between probability distributions $\mathcal{P}$ and $Q$. Formally, $MMD^2(\mathcal{F},P,Q) = ||\vect{\mu_P} - \vect{\mu_Q}||_{\mathcal{F}^2}^2$,
where $\vect{\mu_{\mathcal{P}}}$ and $\vect{\mu_Q} $ are the mean embeddings of $\mathcal{P}$ and $Q$ in a reproducing kernel Hilbert space $\mathcal{F}$.
Given a kernel $\cK$, and samples, $\{ x_1, x_2, \ldots, x_m\} \sim P^m$ and $ \{ x_1', x_2', \ldots, x_k'\} \sim Q^k $, an unbiased estimator for $MMD^2$ can be found in \cite{DBLP:journals/jmlr/GrettonBRSS12, serfling2009approximation}.
\cite{DBLP:conf/iclr/SutherlandTSDRS17} and
\cite{DBLP:journals/jmlr/GrettonBRSS12} used the RBF kernel $\mathcal{K}(x,x') = e^{-\frac{1}{2\sigma^2} ||x-x'||_2^2}$, where $2\sigma^2$ is set to the median of the pairwise Euclidean distances between all samples.
By performing a permutation test on the kernel matrix, the p-value is obtained.
In our experiments (see Section~\ref{sec: Experimental Results}), we thus use four population based baselines: \textbf{KS-Softmax}, \textbf{KS-Embeddings}, \textbf{MMD-Softmax}, and \textbf{MMD-Embeddings}.

As mentioned in the introduction, our work is complementary to the topic of single-instance out-of-distribution (OOD) detection \cite{liang2018enhancing,DBLP:conf/iclr/HendrycksG17,hendrycks2018deep,DBLP:conf/nips/GolanE18, ren2019likelihood, nalisnick2019detecting, DBLP:journals/corr/abs-2106-04015, fort2021exploring}. \guyblue{Although these methods can be applied to each instance in a window, they often fail to capture population statistics within the window, making them inadequate for detecting population-based changes. Additionally, many of these methods are computationally expensive and cannot be applied efficiently to large windows. For example, methods such as those described in~\cite{sastry2020detecting, lee2018simple} extract values from each layer in the network, while others such as~\cite{liang2018enhancing} require gradient calculations. 
We note that the application of the best single-instance methods such as \cite{sastry2020detecting, lee2018simple, liang2018enhancing} in our (large scale) empirical setting is computationally challenging and preclude empirical comparison to our method. Therefore, we consider (in Section~\ref{sec: Experimental Results}) the detection performance of two computationally efficient single-instance baselines: Softmax-Response (abbreviated as \textbf{Single-instance SR} or Single-SR) and Entropy-based (abbreviated as \textbf{Single-instance Ent} or Single-Ent), as described in~\cite{cordella1995method, de2000reject, DBLP:conf/iclr/HendrycksG17}. Specifically, we apply each single-instance OOD detector to every sample in the window and in the detection-training set. We then use a two-sample t-test to determine the p-value between the uncertainty estimators of each sample.}
Finally, we mention \cite{DBLP:conf/nips/GeifmanE17} who developed a risk generalization bound for selective classifiers \cite{JMLR:v11:el-yaniv10a}. The bound presented in that paper is analogous to the coverage generalization bound we present in Theorem~\ref{Theorem:uniform convergence}.
The risk bound in \cite{DBLP:conf/nips/GeifmanE17} can also be used for shift-detection. To apply their risk bound to this task, however, labels, which are not available, are required. Our method (Section~\ref{The method section}) detects distribution shifts without using any labels.

\vspace{-0.2cm}
\label{sec: related work}
Distribution shift detection methods often comprise the following two steps: dimensionality reduction, and a two-sample test between the detection-training sample and test samples.
In most cases, these methods are ``lazy'' in the sense that for each test sample, they make a detection decision based on a computation over the entire detection-training sample. Their performance will be sub-optimal if only a subset of the train sample is used. Figure~\ref{typical pipeline} in
Appendix~\ref{app:Shift Detection General Framework} illustrates this general framework.

The use of dimensionality reduction is optional. It can often improve performance by focusing on a less noisy representation of the data.
Dimensionality reduction techniques include no reduction, \emph{principal components analysis} \cite{wold1987principal}, \emph{sparse random projection} \cite{bingham2001random}, \emph{autoencoders} \cite{rumelhart1985learning, DBLP:conf/nips/PuGHYLSC16}, \emph{domain classifiers}, \cite{Rabanser2019FailingLA} and more. 
In this work we focus on \emph{black box shift detection} (BBSD) methods \cite{DBLP:conf/icml/LiptonWS18}, that rely on deep neural representations of the data generated by a pretrained model.
The representation we extract from the model will typically utilize either the softmax outputs, acronymed BBSD-Softmax, or the embeddings, acronymed BBSD-Embeddings; for simplicity, we may omit the BBSD acronym.
Due to the dimensionality of the final representation, multivariate or multiple univariate two-sample tests can be conducted.

By combining BBSD-Softmax with a Kolmogorov-Smirnov (KS) statistical test \cite{massey1951kolmogorov} and using the Bonferroni correction \cite{bland1995multiple}, 
\cite{Rabanser2019FailingLA} achieved state-of-the-art results in distribution shift detection in the context of image classification (MNIST and CIFAR-10). We acronym their method as KS-Softmax. The \emph{univariate} KS test processes individual dimensions separately; its statistic is calculated by computing the largest difference $Z$ of the \emph{cumulative density functions} (CDFs) across all dimensions as follows: $    Z = \displaystyle  \sup_{\vect{z}}\lvert F_P(\vect{z}) - F_Q(\vect{z}) \rvert $,
where $F_Q$ and $F_P$ are the empirical CDFs of the detection-training and test data (which are sampled from $P$ and $Q$, respectively; see Section~\ref{sec:Problem Formulation}). The Bonferroni correction rejects the null hypothesis when the minimal p-value among all tests is less than $\frac{\alpha}{d}$, where $\alpha$ is the significance level of the test, and $d$ is the number of dimensions. 
Although there have been several less conservative approaches to aggregation \cite{heard2018choosing, loughin2004systematic}, they usually assume some dependencies among the tests.

The \emph{maximum mean discrepancy} (MMD) method \cite{DBLP:journals/jmlr/GrettonBRSS12} is a kernel-based multivariate test that can be used to distinguish between probability distributions $P$ and $Q$. Formally, $MMD^2(\mathcal{F},P,Q) = ||\vect{\mu_P} - \vect{\mu_Q}||_{\mathcal{F}^2}^2$,
where $\vect{\mu_P}$ and $\vect{\mu_Q} $ are the mean embeddings of $P$ and $Q$ in a reproducing kernel Hilbert space $\mathcal{F}$.
Given a kernel $\cK$, and samples, $\{ x_1, x_2, \ldots, x_m\} \sim P^m$ and $ \{ x_1', x_2', \ldots, x_k'\} \sim Q^k $, an unbiased estimator for $MMD^2$ can be found in \cite{DBLP:journals/jmlr/GrettonBRSS12, serfling2009approximation}.
\cite{DBLP:conf/iclr/SutherlandTSDRS17} and
\cite{DBLP:journals/jmlr/GrettonBRSS12} used the RBF kernel $\mathcal{K}(x,x') = e^{-\frac{1}{2\sigma^2} ||x-x'||_2^2}$, where $2\sigma^2$ is set to the median of the pairwise Euclidean distances between all samples.
By performing a permutation test on the kernel matrix, the p-value is obtained.
In our experiments (see Section~\ref{sec: Experimental Results}), we thus use four population based baselines: \textbf{KS-Softmax}, \textbf{KS-Embeddings}, \textbf{MMD-Softmax}, and \textbf{MMD-Embeddings}.

As mentioned in the introduction, our work is complementary to the topic of single-instance out-of-distribution (OOD) detection \cite{liang2018enhancing,DBLP:conf/iclr/HendrycksG17,hendrycks2018deep,DBLP:conf/nips/GolanE18, ren2019likelihood, nalisnick2019detecting, DBLP:journals/corr/abs-2106-04015, fort2021exploring}. \guyblue{Although these methods can be applied to each instance in a window, they often fail to capture population statistics within the window, making them inadequate for detecting population-based changes. Additionally, many of these methods are computationally expensive and cannot be applied efficiently to large windows. For example, methods such as those described in~\cite{sastry2020detecting, lee2018simple} extract values from each layer in the network, while others such as~\cite{liang2018enhancing} require gradient calculations. 
We note that the application of the best single-instance methods such as \cite{sastry2020detecting, lee2018simple, liang2018enhancing} in our (large scale) empirical setting is computationally challenging and preclude empirical comparison to our method. Therefore, we consider (in Section~\ref{sec: Experimental Results}) the detection performance of two computationally efficient single-instance baselines: Softmax-Response (abbreviated as \textbf{Single-instance SR} or Single-SR) and Entropy-based (abbreviated as \textbf{Single-instance Ent} or Single-Ent), as described in~\cite{cordella1995method, de2000reject, DBLP:conf/iclr/HendrycksG17}. Specifically, we apply each single-instance OOD detector to every sample in the window and in the detection-training set. We then use a two-sample t-test to determine the p-value between the uncertainty estimators of each sample.}

 Finally, we mention \cite{DBLP:conf/nips/GeifmanE17} who developed a risk generalization bound for selective classifiers \cite{JMLR:v11:el-yaniv10a}. The bound presented in that paper is analogous to the coverage generalization bound we present in Theorem~\ref{Theorem:uniform convergence}.
The risk bound in \cite{DBLP:conf/nips/GeifmanE17} can also be used for shift-detection. To apply their risk bound to this task, however, labels, which are not available, are required. Our method (Section~\ref{The method section}) detects distribution shifts without using any labels.

\section{Proposed Method -- Coverage-Based Detection}
\label{The method section}
In this section, we present a novel technique for detecting a distribution shift based on selective prediction principles (definitions follow). We develop a tight generalization coverage bound that holds with high probability for ID data, sampled from the source distribution. The main idea is that violations of this coverage bound indicate a distribution shift with high probability.

\subsection{Selection with Guaranteed Coverage}
\label{sec:Selection with Guaranteed Lower Coverage}
We begin by introducing basic selective prediction terminology and definitions that 
are required to describe our method.
Consider a standard multiclass classification problem, where $\mathcal{X}$ is some feature space (e.g., raw image data) and $\mathcal{Y}$ is a finite label set, $\mathcal{Y} = \{1,2,3,...,C\}$, representing $C$ classes. Let $P(X,Y)$ be a probability distribution over $\mathcal{X} \times \mathcal{Y}$, and define a
\emph{classifier} as a function $f: \mathcal{X} \rightarrow \mathcal{Y}$. We refer to $P$ as the \emph{source distribution}. A \emph{selective classifier} \cite{JMLR:v11:el-yaniv10a} is a pair $(f,g)$, where $f$ is a classifier and $g: \mathcal{X} \rightarrow \{ 0, 1 \}$ is a \emph{selection function} \cite{el2010foundations}, which serves as a binary qualifier for $f$ as follows,
\vspace{-0.2cm}
\[ (f,g)(x) \triangleq
\begin{cases}
    f(x),& \text{if } g(x) = 1;\\
    \text{don't know},  & \text{if } g(x) = 0.
\end{cases} \]

 \begin{algorithm}[H]

   \caption{\emph{Selection with guaranteed coverage} \textbf{(SGC)}}
    \label{Alg:CD}
     \textbf{Input:} detection-training set: $S_m$, confidence-rate function: $\kappa_f$, confidence parameter $\delta$, target coverage: $c^*$.
    

    Sort $S_m$ according to $\kappa_f(x_i)$, $x_i \in S_m$ (and now assume w.l.o.g. that indices reflect this ordering).
    
    $z_{\text{min}} = 1$, $z_{\text{max}} = m$
    
      \For{$ i=1 $ \KwTo $k= \ceil{ \log_2 m }$}{
      
      $z = \ceil{(z_{\text{min}} + z_{\text{max}}) / 2}$
      
      $\theta_i = \kappa_f(x_z)$
      
      Calculate $\hat{c_i}(\theta_i, S_m)$
      
      Solve for $b^*_i (m,m \cdot \hat{c}_i(\theta_i, S_m), \frac{\delta}{k})$ \{see Lemma \ref{lemma}\}
      
        \eIf{$b^*_i(m, m \cdot \hat{c}_i(\theta_i, S_m),             \frac{\delta}{k}) \leq c^*$}{
            $z_{\text{max}} = z$
        }{
            $z_{\text{min}} = z$
        } 

      }
    \textbf{Output:} bound: $b^*_k(m, m \cdot \hat{c}_k(\theta_k, S_m),             \frac{\delta}{k})$, threshold: $\theta_k$.

\end{algorithm}

\vspace{-0.2cm}
A general approach for constructing a
selection function based on a given classifier $f$ is to work in terms of a \emph{confidence-rate
function} \cite{geifman2018biasreduced}, $\kappa_f : \mathcal{X} \rightarrow \mathbb{R}^{+}$, referred to as CF.
The CF $\kappa_f$ should quantify confidence in predicting the label of  $x$ 
based on signals extracted from $f$ \cite{geifman2018biasreduced}.
The most common and well-known CF for a classification model $f$ (with softmax at its last layer) is its \emph{softmax response} (SR) value \cite{cordella1995method, de2000reject, DBLP:conf/iclr/HendrycksG17}. 
A given CF $ \kappa_f $ can be straightforwardly used to define a selection function: $g_{\theta}(x) \triangleq g_\theta (x|\kappa_f) = \mathds{1}[\kappa_f(x) \geq \theta]$,
where $\theta$ is a user-defined constant.
For any selection function, we define its \emph{coverage} w.r.t. a distribution $\mathcal{P}$ (recall, $\mathcal{P} \triangleq P_X$, see Section \ref{sec:Problem Formulation}), and its \emph{empirical coverage} w.r.t. a sample $S_k \triangleq \{ x_1, x_2, \ldots x_k \}$, as $c(\theta,\mathcal{P}) \triangleq \mathop{\mathbb{E}}_{\mathcal{P}}[g_{\theta}(x)]$, and $\hat{c}(\theta, S_k) \triangleq \frac{1}{k}\sum_{i=1}^k g_{\theta}(x_i)$, respectively.

Given a bound on the expected coverage for a given selection function, we can use it to detect a distribution shift via violations of the bound. 
We now develop such a bound and show how to use it to detect distribution shifts. For a classifier $f$, a detection-training sample $S_m \sim \mathcal{P}^m$, a confidence parameter $\delta > 0$, and a desired coverage $c^* > 0$, our goal is to use $S_m$ to find a $\theta$ value (which implies a selection function $g_\theta$) that guarantees the desired coverage. This means that \emph{under coverage} should occur with probability of at most $ \delta $,
\vspace{-0.2cm}
\begin{equation}
\label{coverage guerantee eq}
\guyPr_{S_m} \{ c(\theta, P) < c^* \} < \delta.
\end{equation}

\vspace{-0.3cm}
A $\theta$ that guarantees Equation~(\ref{coverage guerantee eq}) provides a probabilistic lower bound, guaranteeing that coverage $c$ of ID unseen population (sampled from $\mathcal{P}$) satisfies $c > c^*$ with probability of at least $1-\delta$.
For the remaining of this section, we introduce the \emph{selection with guaranteed coverage} (SGC) algorithm, which is utilized for constructing a lower bound. 


The SGC algorithm receives as input a classifier $f$, a CF $\kappa_f$, a confidence parameter $\delta$, a target coverage $c^*$, and a detection-training set $S_m$. The algorithm performs a binary search to find the optimal coverage lower bound with confidence $\delta$, and outputs a coverage bound $ b^* $ and the threshold $ \theta $, defining the selection function.
A pseudo code of the SGC algorithm appears in 
Algorithm~\ref{Alg:CD}.

Our analysis of the SGC algorithm makes use of Lemma~\ref{lemma}, which gives 
a tight numerical (generalization) bound on the expected coverage, based on a test over a sample. The proof of Lemma~\ref{lemma} is nearly identical to Langford's proof of Theorem 3.3 in \cite{langford2005tutorial}, p. 278, where instead of the empirical error used in \cite{langford2005tutorial}, we use the empirical coverage, which is also a Bernoulli random variable.
    
    
      
      
      
      
      



\begin{lemma}
    \label{lemma}
    Let $\mathcal{P}$ be any distribution and consider a selection function $g_\theta$ with a threshold $\theta$ whose coverage is $c(\theta,\mathcal{P})$. Let $0 < \delta < 1$ be given and let $\hat{c}(\theta, S_m)$ be the empirical coverage w.r.t. the set $S_m$, sampled i.i.d. from $\mathcal{P}$. Let $b^*(m, m \cdot \hat{c}(\theta, S_m), \delta)$  be the solution of the following equation:
            \begin{equation}
        \label{Bin}
        \underset{b}{\arg\min} \left( \sum_{j=0}^{m \cdot \hat{c}(\theta, S_m)} \binom{m}{j} b^{j} (1-b)^{m-j} \leq 1-\delta \right).
        \end{equation}
    \vspace{-0.2cm}
    Then,
    \begin{equation}
         \label{prob}
        \guyPr_{S_m} \{c(\theta,\mathcal{P}) < b^*(m,m \cdot  \hat{c}(\theta, S_m), \delta) \} < \delta.
    \end{equation}
\end{lemma}

The following is a uniform convergence theorem for the SGC procedure stating that all the calculated bounds are valid simultaneously with a probability of at least $1 - \delta$. 

\begin{Theorem} \textbf{(SGC -- Uniform convergence)}
\label{Theorem:uniform convergence}
Assume $S_m$ is sampled i.i.d. from $\mathcal{P}$, and consider an application of Algorithm~\ref{Alg:CD}. For $k = \ceil{\log_2{m}}$, let $b^*_i(m, m \cdot \hat{c}_i(\theta_i, S_m), \frac{\delta}{k})$ and $\theta_i$ be the values obtained in the $\text{i}^{\text{th}}$ iteration of Algorithm~\ref{Alg:CD}. Then, 
\begin{center}
     $\guyPr_{S_m} \{\exists i : c(\theta_i,\mathcal{P}) < b^*_i(m, m\cdot \hat{c}_i(\theta_i, S_m), \frac{\delta}{k}) \} < \delta$.
\end{center}
\end{Theorem}

\begin{proof}[Proof (sketch - see full proof in the Appendix~\ref{Proof:Theorem:uniform convergence})]
Define, $\newline {\cal B}_{\theta_i} \triangleq b^*_i(m, m \cdot \hat{c}_i(\theta_i, S_m), \frac{\delta}{k})$, ${\cal C}_{\theta_i} \triangleq c(\theta_i,\mathcal{P})$, then,
\vspace{-0.3cm}
\begin{eqnarray}
    \guyPr_{S_m} \{\exists i : {\cal C}_{\theta_i} < {\cal B}_{\theta_i} \} &=& \sum_{i=1}^k \int_{0}^{1} \,d\theta'  \guyPr_{S_m} \{{\cal C}_{\theta'} < {\cal B}_{\theta'}\} \cdot \guyPr_{S_m} \{ \theta_i = \theta' \}  \nonumber \\ 
    &<& \sum_{i=1}^k \int_{0}^{1} \,d\theta' \frac{\delta}{k} \cdot \guyPr_{S_m} \{ \theta_i = \theta' \} = \sum_{i=1}^k \frac{\delta}{k} = \delta. \nonumber
\end{eqnarray}

\vspace{-0.7cm}
\end{proof}

\vspace{-0.4cm}
\subsection{\guyblue{Coverage-Based Detection Algorithm}}
\label{sec: Coverage-Based Detection Algorithm}

Our detection algorithm applies SGC to $C_{\text{target}}$ target coverages uniformly spread between the interval $[0.1,1]$, excluding the coverage of 1. We set $C_{\text{target}}=10, \delta = 0.01$, and $\kappa_f(x) = 1 - \text{Entropy}(x)$ for all our experiments; our method appears to be robust to those hyper-parameters, as we demonstrate in Appendix~\ref{app: Ablation Study}. Each application $j$ of SGC on the same sample $S_m \sim P^m$ with a target coverage of $c_j^*$ produces a pair: $(b^{*}_j, \theta_j)$, which represent a bound and a threshold, respectively. We define $\delta(\hat{c}|b^*)$ to be a binary function that indicates a bound violation, where $\hat{c}$ is the empirical coverage (of a sample) and $b^*$ is a bound, $\delta(\hat{c}|b^*) = \mathds{1}[\hat{c} <= b^*]$.
Thus, given a window of $k$ samples from an alternative distribution, $W_k \sim Q^k$, we define the 'sum bound violations', $V$, as follows,
\small
\begin{eqnarray}
\label{eq: V}
    V = \frac{1}{C_{\text{target}}}\sum_{j=1}^{C_{\text{target}}} \big(b^{*}_j - \hat{c}(\theta_j, W_k)\big) \cdot \delta(\hat{c}(\theta_j, W_k)|b^{*}_j)= \frac{1}{k \cdot C_{\text{target}}}\sum_{j=1}^{C_{\text{target}}} \sum_{i=1}^k \big( b^{*}_j - g_{\theta_j}(x_i)\big) \cdot \delta(\hat{c}_j|b^{*}_j),
\end{eqnarray}
\normalsize
where we obtain the last equality by using, $\hat{c}_j \triangleq \hat{c}(\theta_j, W_k) = \frac{1}{k} \sum_{i=1}^k g_{\theta_j}(x_i)$.
Considering Figure~\ref{fig: test window}, the quantity $V$ is the sum of distances from the deviation violations (red dots) to the linear diagonal representing the coverage bounds (black line).
 

When $Q$ equals $\mathcal{P}$, which implies that there is no distribution shift, we expect that all the bounds (computed by SGC over $S_m$) will hold over $W_k$, namely $\delta(\hat{c}_j|b^{*}_j) = 0$ for every iteration $j$; in this case, $V = 0$. Otherwise, $V$ will indicate the violation magnitude.
Since $V$ represents the average of $k \cdot C_{\text{target}}$ values (if at least one bound violation occurs), for cases where $k \cdot C_{\text{target}} \gg 30$ (as in all our experiments), we can assume that $V$ follows a nearly normal distribution \cite{james2013introduction} and perform a t-test\footnote{In our experiments, we use SciPy's {\tt stats.ttest\_1samp} implementation \cite{-NMeth2020SciPy} for the t-test.} 
to test the null hypothesis, $H_0: V = 0$, against the alternative hypothesis, $H_1: V \geq 0$. The null hypothesis is rejected if the p-value is less than the specified significance level $\alpha$. A pseudocode of our \emph{coverage-based detection algorithm} appears in Algorithm~\ref{Alg:Detection}.

\begin{algorithm}[H]
  \caption{\emph{\textbf{Coverage-Based Detection}}} 
    \label{Alg:Detection}

    \SetKwInOut{KwInt}{Input Training}
    \SetKwInOut{KwInd}{Input Detection}
    \SetKwInOut{KwOutt}{Output Training}
    \SetKwInOut{KwOutd}{Output Detection}
    \SetKwInOut{Ret}{Return}
    \SetKwInOut{Output}{Output}
    \tcp{Fit}
    \textbf{Input Training:} $S_m$, $\delta$, $\kappa_f$, $C_{\text{target}}$
    
    Generate $C_{\text{target}}$ uniformly spread coverages $\vect{c}^*$
    
    \For{$ j=1 $ \KwTo $C_{\text{target}}$}{
     $b^{*}_j$, $\theta_j$ $=$ SGC($S_m$, $\delta$,  $c_j^*$, $\kappa_f$)
     
    }
    \textbf{Output Training:} $ \{ (b^{*}_j, \theta_j) \}_{j=1}^{C_{\text{target}}}$
    
    \tcp{Detect}
    \textbf{Input Detection:} $ \{ (b^{*}_j, \theta_j) \}_{j=1}^{C_{\text{target}}}$, $\kappa_f$, $\alpha$, $k$
    
    \While{\tt{True}}{
    
    Receive window $W_k = \{ x_1, x_2, \ldots, x_k \}$
    

    
    Calculate $V$ \{see Eq~(\ref{eq: V})\}
     
     Obtain p-value from t-test, $H_0: V = 0$, $H_1: V \geq 0$
    
    \If{$p_{\text{value}} < \alpha$}{
    
    {\tt Shift\_detected} $\leftarrow$ {\tt True}
    
    \textbf{Output Detection:} {\tt Shift\_detected}, $p_{\text{value}}$
    }
    
    }
\end{algorithm}

To fit our detector, we apply SGC (Algorithm~\ref{Alg:CD}) on the detection-training set, $S_m$, for $ C_{\text{target}}$ times, in order to construct the pairs $ \{ (b^{*}_j, \theta_j) \}_{j=1}^{C_{\text{target}}}$. Our detection model utilizes these pairs to monitor a given model, receiving at each time instant a test sample window of size $k$ (user defined), $W_k = \{ x_1, x_2, \ldots, x_k \}$, which is inspected to see if its content is distributionally shifted from the underlying distribution reflected by the detection-training set $S_m$.


Our approach encodes all necessary information for detecting distribution shifts using only $C_{\text{target}}$ scalars (in our experiments, we set $C_{\text{target}}=10$), which is independent of the size of $S_m$. In contrast, the baselines process the detection-training set, $S_m$, which is typically very large, for every detection they make. This makes our method significantly more efficient than the baselines, see Figure~\ref{fig: complexity} and Table~\ref{tab:complexity table}, in Section~\ref{sec: Complexity Anylysis}.

\section{Complexity Analysis}
\label{sec: Complexity Anylysis}
\guyblue{This section provides a complexity analysis of our method as well as the baselines mentioned in Section~\ref{sec: related work}. Table~\ref{tab:complexity table} summarizes the complexities of each approach, and Figure~\ref{fig: complexity} shows their run-time (in seconds) as a function of the detection-training set size, denoted as $m$.}

All baselines are lazy learners (analogous to nearest neighbors) in the sense that they require the entire source (detection-training) set for each detection decision they make. Using only a subset will result in sub-optimal performance, since it might not capture all the necessary information within the source distribution, $\mathcal{P}$.
\begin{table}[H]    
  \centering 
    \begin{adjustbox}{width=\textwidth}
    \begin{tabular}{lccc}

    \caption{Complexity comparison. \textbf{bold} entries indicate the best detection complexity. $m$, $k$ refers to the detection-training size (or source size), and window size, respectively. $d$ refers to the number of dimensions after dimensionality reduction.}
    \label{tab:complexity table}
    \textbf{Detection Method} &
    \textbf{Space} &
    \textbf{Time} \\
    
    \midrule\midrule
    
    \textbf{Coverage-Based Detection} (Ours)  & \boldmath{$ O(k)$}       & \boldmath{$ O(k)$}        \\
    \textbf{MMD}       & $O( m^2 + k^2 + mk ) $ & $O\bigl( d( m^2 + k^2 + mk ) \bigr)$\\     
    \textbf{KS}       & $ O \bigl(d( m  + k ) \bigr) $        & $O \bigl( d(  m \log m + k \log k ) \bigr) $ \\

    \textbf{Single-instance}       & $ O ( m  + k ) $        & $ O (m  + k ) $ \\
    \midrule\midrule
    
    \end{tabular}
    \end{adjustbox}
\end{table}

In particular, MMD is a permutation test \cite{DBLP:journals/jmlr/GrettonBRSS12} that also employs a kernel. The complexity of kernel methods is dominated by the number of instances and, therefore, the time and space complexities of MMD are $O(d(m^2 + k^2 + mk))$ and $O(m^2 + k^2 + mk)$, respectively, where in the case of DNNs, $d$ is the dimension of the embedding or softmax layer used for computing the kernel, and $k$ is the window size.
The KS test \cite{massey1951kolmogorov} is a univariate test, which is applied on each dimension separately and then aggregates the results via a Bonferroni correction. Its time and space complexities are $O(d(m\log m + k\log k))$ and $O(d(m+k))$, respectively.

The single-instance baselines simply conduct a t-test on the heuristic uncertainty estimators (SR or Entropy-based) between the detection-training set and the window data, resulting in a time and space complexity of $O(m+k)$.
\begin{figure}[H]
\begin{center}
\includegraphics[width=\textwidth]{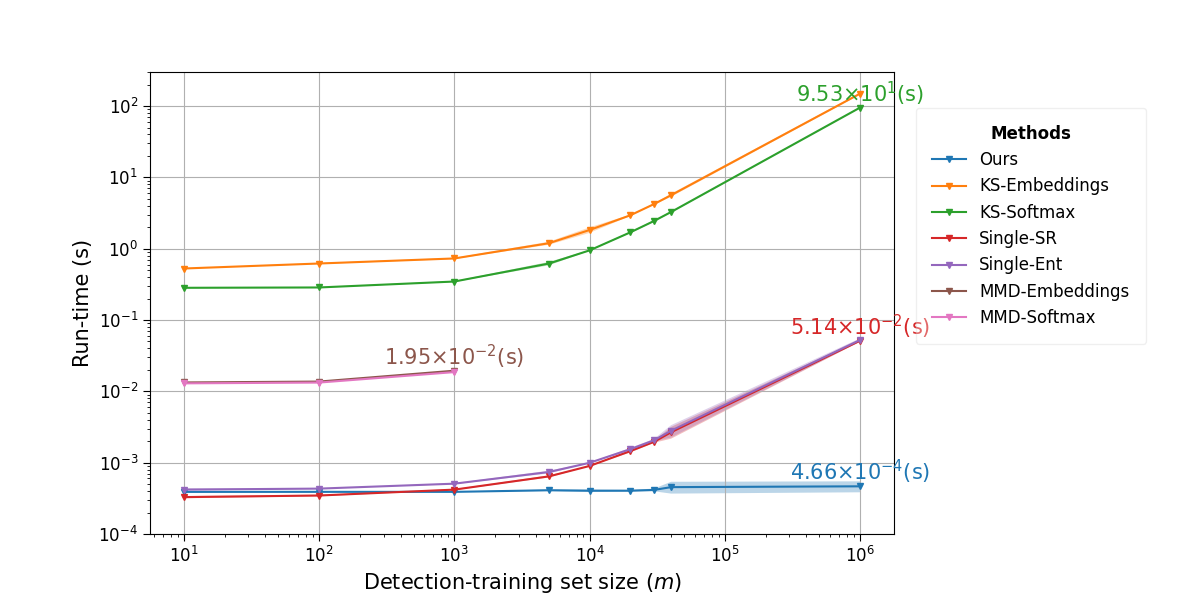}
\caption{Our method outperforms the baselines in terms of scalability, with a significant five orders of magnitude improvement in run-time compared to the best baseline, KS-Softmax, when using a detection-training set of $m=1,000,000$ samples. One-$\sigma$ error-bars are shadowed.}
\label{fig: complexity}
\end{center}
\end{figure}

\guyblue{The fitting procedure of our coverage-based detection algorithm incurs time and space complexities of $O(m\log m)$ and $O(m)$, respectively.
Each subsequent detection round incurs time and space complexities of $O(k)$, which is independent of the size of the detection-training (or source) set. Our method's superior efficiency is demonstrated both theoretically and empirically, as shown in Table~\ref{tab:complexity table} and Figure~\ref{fig: complexity}, respectively. Figure~\ref{fig: complexity} shows the run-time (in seconds) for each detection method\footnote{\label{ftn: kernel}MMD detection-training set is capped at 1,000 due to kernel method's size dependence.} as a function of the detection-training set size, using a ResNet50 and a fixed window size of $k=10$ samples. Our detection time remains constant regardless of the size of the detection-training set, as opposed to the baseline methods, which exhibit a run-time that grows as a function of the detection-training set size.
In particular, our method achieves a five orders of magnitude improvement in run-time over the best performing baseline (as we demonstrate in Section~\ref{sec: Experimental Results}), KS-Softmax, when the detection-training set size is 1,000,000\footnote{This specific experiment is simulated (since the ImageNet validation set size is 50,000).}. Specifically, our detection time is $4.66 \cdot 10^{-4}$ seconds, while KS-Softmax's detection time is $95.3$ seconds, rendering it inapplicable for real-world applications.}

\section{\guyblue{Experiments - Detecting Distribution Shifts}}
\label{sec: Experiments - Detecting Distribution Shifts}

In this section we showcase the effectiveness of our method, along with the considered baselines, in detecting distribution shifts. 

\subsection{Setup}
\label{sec: Setup}

Our experiments are conducted on the ImageNet dataset~\cite{DBLP:conf/cvpr/DengDSLL009}, using its validation dataset as proxies for the source distribution, $\mathcal{P}$. We utilized three well-known architectures, ResNet50~\cite{DBLP:conf/cvpr/HeZRS16}, MobileNetV3-S~\cite{howard2019searching}, and ViT-T~\cite{dosovitskiy2020image}, all of which are publicly available in timm's repository~\cite{rw2019timm}, as our pre-trained models. To train our detectors, we randomly split the ImageNet validation data (50,000) into two sets, a detection-training or source set, which is used to fit the detectors\footref{ftn: kernel} (49,000) and a validation set (1,000) for applying the shift. To ensure the reliability of our results, we repeated the shift detection performance evaluation on 15 random splits, applying the same type of shift to different subsets of the data. Inspired by \cite{Rabanser2019FailingLA}, we evaluated the models using various window sizes, $|W_k| \in \{ 10, 20, 50, 100, 200, 500, 1000 \}$.

\subsubsection{Distribution Shift Datasets}
\label{sec: Distribution Shift Datasets}
At test time, the 1,000 validation images obtained via our split can be viewed as the in-distribution
(positive) examples. For out-of-distribution (negative) examples, we follow the common setting for detecting OOD samples or distribution shifts~\cite{Rabanser2019FailingLA, DBLP:conf/iclr/HendrycksG17, liang2018enhancing}, and test on several different natural image datasets and synthetic noise datasets.
More specifically, we investigate the following types of shifts:

(1) \textbf{Adversarial via FGSM}: We transform samples into adversarial ones using the \emph{Fast Gradient Sign Method} (FGSM)~\cite{43405}, with $\epsilon \in \{ 7 \cdot 10^{-5}, 1 \cdot 10^{-4}, 3 \cdot 10^{-4}, 5 \cdot 10^{-4} \}$. (2) \textbf{Adversarial via PGD}: We convert samples into adversarial examples using \emph{Projected Gradient Descent}  (PGD)~\cite{madry2018towards}, with $\epsilon=1 \cdot 10^{-4}$; we use 10 steps, $\alpha = 1 \cdot 10^{-4}$, and random initialization. (3) \textbf{Gaussian noise}: We corrupt test samples with Gaussian noise using standard deviations of $\sigma \in \{ 0.1, 0.3, 0.5, 1 \}$. (4) \textbf{Rotation}: We apply image rotations, $\theta \in \{ 5^{\circ}, 10^{\circ}, 20^{\circ}, 25^{\circ} \}$. (5) \textbf{Zoom}: We corrupt test samples by applying zoom-out percentages of $\{ 90\% , 70\% , 50\% \}$. (6) \textbf{ImageNet-O}: We use the ImageNet-O dataset~\cite{hendrycks2021nae}, consisting of natural out-of-distribution samples. (7) \textbf{ImageNet-A}: We use the ImageNet-A dataset~\cite{hendrycks2021nae}, consisting of natural adversarial samples. (8) \textbf{No-shift}: We include the no-shift case to evaluate false positives. To determine the severity of the distribution shift, we refer the reader to Table~\ref{tab: Acc} in Appendix~\ref{app: Exploring Model Sensitivity: Evaluating Accuracy on Shifted Datasets}.











\subsection{Maximizing Detection Power Through Lower Coverages}
\begin{figure}[H]

\begin{subfigure}[H]{0.4\linewidth}
\includegraphics[width=\linewidth]{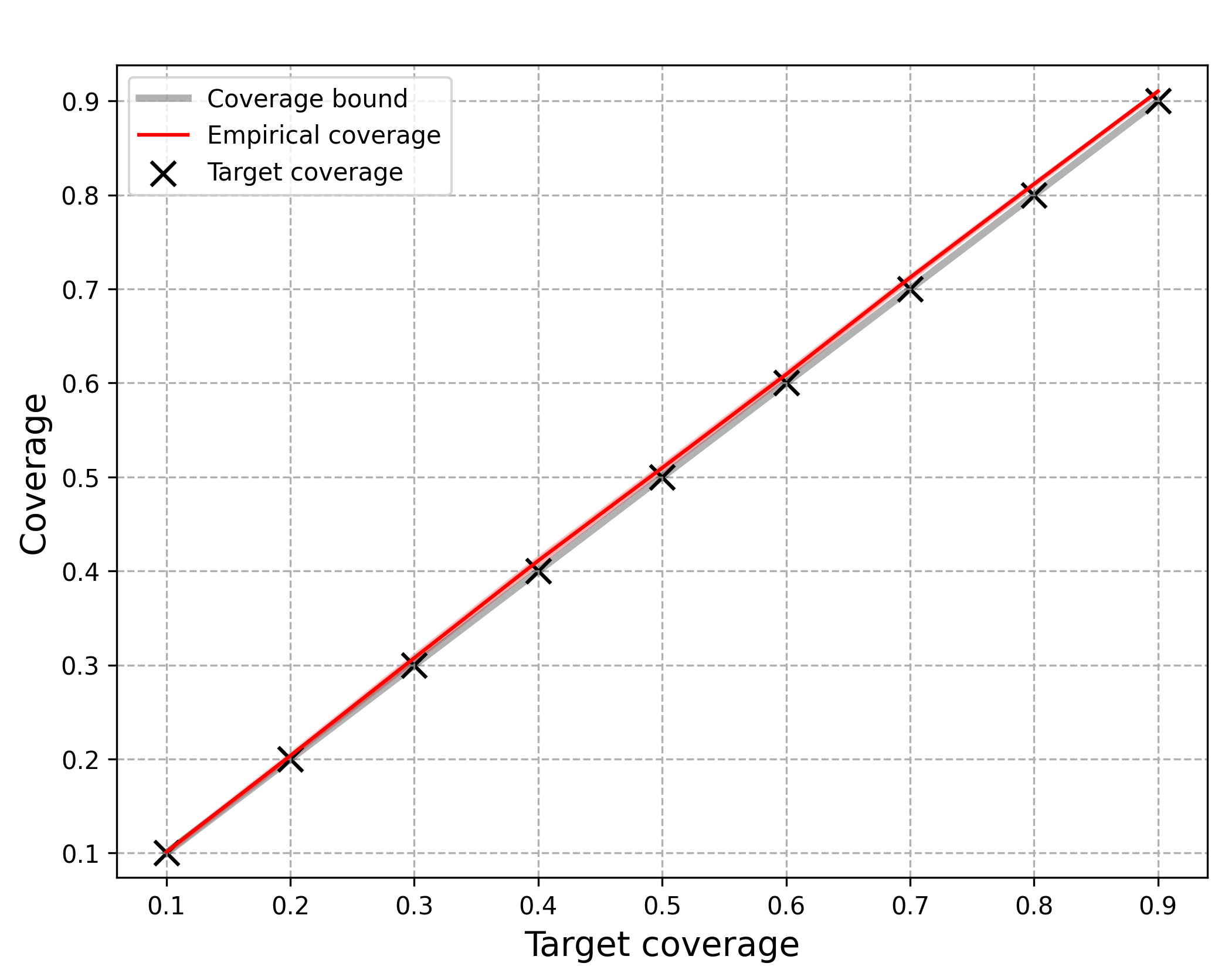}
\caption{No-shift case.}
\label{fig: val window}
\end{subfigure}
\begin{subfigure}[H]{0.4\linewidth}
\includegraphics[width=\linewidth]{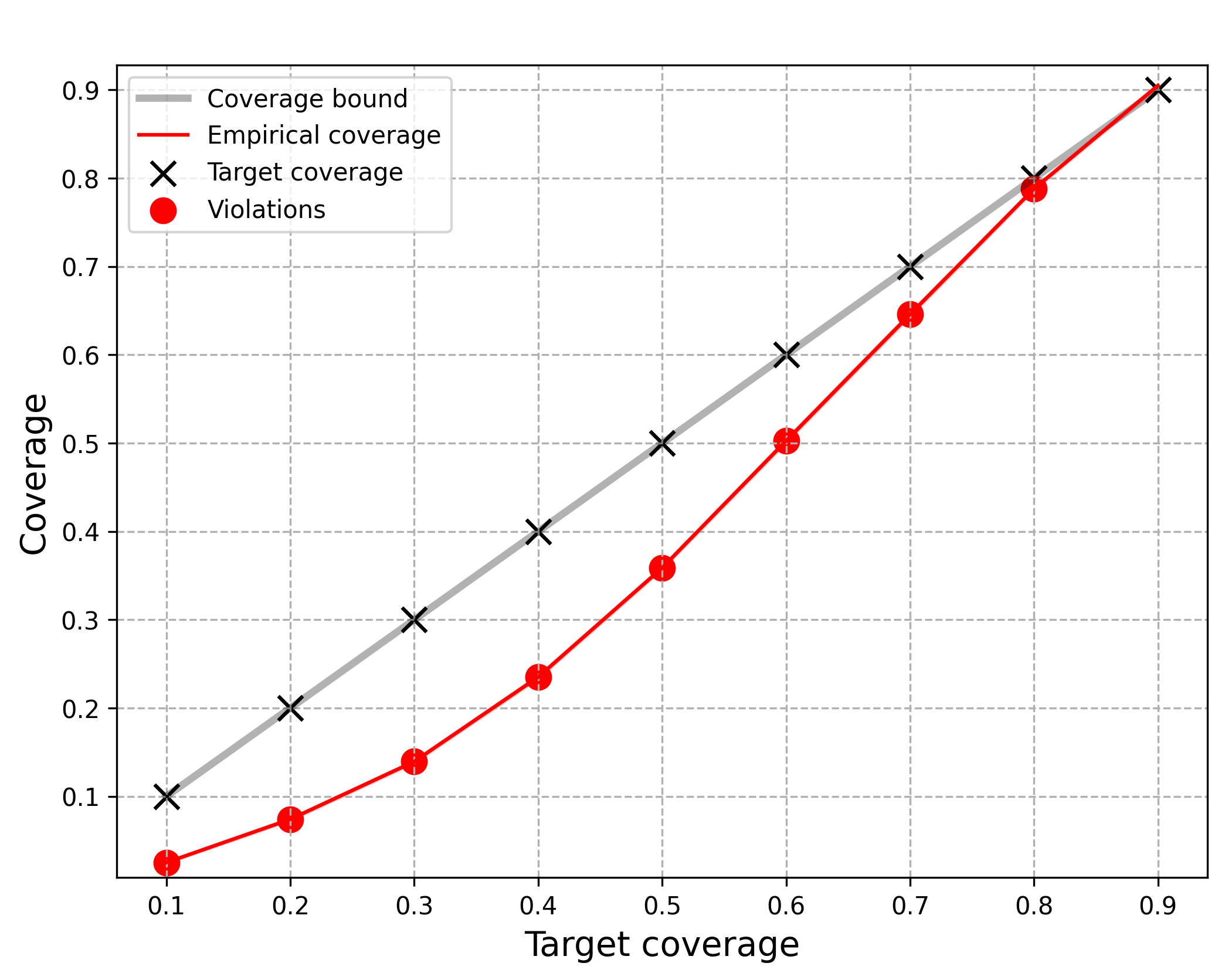}
\caption{Shift case.}
\label{fig: test window}
\end{subfigure}%
\caption{Visualization of our detection algorithm, Algorithm~\ref{Alg:Detection}.}
\label{fig: windows}
\end{figure}

In this section, we provide an intuitive understanding of our proposed method, as well as a demonstration of the importance of considering lower coverages, when detecting population-based distribution shifts. For all experiments in this section, we employed a ResNet50~\cite{DBLP:conf/cvpr/HeZRS16} and used a detection-training set, which was randomly sampled from the ImageNet validation set.

We validate the effectiveness of our proposed detection method (Algorithm~\ref{Alg:Detection}) by demonstrating its performance on two distinct scenarios, the no-shift case and the shift case. In the no-shift scenario, as depicted in Figure~\ref{fig: val window}, we randomly selected a 1,000 sample window from the ImageNet dataset that did not overlap with the detection-training set. We then tested whether this window was distributionally shifted using our method. The shift scenario, as depicted in Figure~\ref{fig: test window}, involved simulating a distributional shift by randomly selecting 1,000 images from the ImageNet-O dataset. Each figure showcases the target coverages of each application, the coverage bounds provided by SGC, the empirical coverage of each sample (for each threshold), and the bound violations (indicated by the red circles), when they occur. 


Both figures show that the target coverages and the bounds returned by SGC are essentially identical. Consider the empirical coverage of each threshold returned by SGC for each sample. In the no-shift case, illustrated in Figure~\ref{fig: val window}, it is apparent that all bounds are tightly held; demonstrated by the fact that each empirical coverage is slightly above its corresponding bound. In contrast, the shift case, depicted in Figure~\ref{fig: test window}, exhibits bound violations, indicating a noticeable distribution shift. Specifically, the bound holds at a coverage of 0.9, but violations occur at lower coverages, with the maximum violation magnitude occurring around a coverage of 0.4. This underscores the significance of using coverages lower than 1, and suggests that the detection power might lay within lower coverages.




\vspace{-0.2cm}
\subsection{Evaluation Metrics}
\vspace{-0.2cm}
\label{sec: Evaluation Metrics}
In order to benchmark our results against the baselines, we adopt the following metrics to measure the effectiveness in distinguishing between in- and out-of-distribution windows:
(1)  \textbf{Area Under the Receiver Operating Characteristic} $ \uparrow $  (AUROC) is a threshold-independent metric \cite{davis2006relationship}. The ROC curve illustrates the relationship between the \emph{true positive rate} (TPR) and the \emph{false positive rate} (FPR), and the AUROC score represents the probability that a positive example will receive a higher detection score than a negative example \cite{fawcett2006introduction}. A perfect detector would achieve an AUROC of 100\%. (2) \textbf{Area Under the Precision-Recall} $ \uparrow $  (AUPR) is another threshold-independent metric \cite{manning1999foundations, saito2015precision}. The PR curve is a graph that depicts the relationship between \emph{precision} (TP / (TP + FP)) and \emph{recall} (TP / (TP + FN)). AUPR-In and AUPR-Out in Table~\ref{tab: main results} represent the area under the precision-recall curve where in-distribution and out-of-distribution images are considered as positives, respectively. (3) \textbf{False positive rate (FPR) at 95\% true positive rate (TPR)} $ \downarrow $, denoted as FPR@95TPR, is a performance metric that measures the FPR at a specific TPR threshold of 95\%. This metric is calculated by finding the FPR when the TPR is 95\%. (4)  \textbf{Detection Error} $ \downarrow $ is the probability of misclassification when TPR is fixed at 95\%. It's a weighted average of FPR and complement of TPR, given by $\text{P}_{\text{e}} = 0.5(1 - TPR) + 0.5FPR$. Assuming equal probability of positive and negative examples in the test set. $ \uparrow $ indicates larger value is better, and $ \downarrow $ indicates lower values is better.

\vspace{-0.2cm}
\subsection{Experimental Results}
\label{sec: Experimental Results}
\vspace{-0.2cm}
In this section we demonstrate the effectiveness of our proposed method, as well as the considered baselines. To ensure a fair comparison, as mentioned in Section~\ref{sec: Setup} ,we utilize the ResNet50, MobileNetV3-S, and ViT-T architectures, which are pretrained on ImageNet 
and their open-sourced code appear at timm~\cite{rw2019timm}. We provide code for reproducing our experiments\footref{ftn: code}.






Table~\ref{tab: main results} presents a comprehensive summary of our main results, which clearly demonstrate that our proposed method outperforms all baseline methods in the majority of cases (combinations of model architecture, window size, and metrics considered), thus providing a clear indication of the effectiveness of our approach. Our method demonstrates the highest effectiveness in the low to mid window size regime. Specifically, for window sizes of 50 samples, our method outperforms all baselines across all architectures, with statistical significance for all considered metrics, by a large margin. This observation is true for window sizes of 100 samples as well, except for the MobileNet architecture, where KS-Softmax shows slightly better performance in the metrics of Detection Error and TNR@95TPR. The second-best performing method overall appears to be KS-Softmax (consistent with the findings of~\cite{Rabanser2019FailingLA}), which seems to be particularly effective in large window sizes, such as 500 or more. Our method stands out for its consistent performance across different architectures. In particular, our method achieves an AUROC score of over 90\% across all architectures tested when using a window size of 50,
while the second best performing method (KS-Softmax), fails to achieve such a score even with a window size of 100 samples, when using the ResNet50 architecture.

It is evident from Table~\ref{tab: main results}, that as far as population-based detection tasks are concerned,
neither of the single-instance baselines (SR and Entropy) offers an effective solution. These baselines show inconsistent performance across the various architectures considered. For instance, when ResNet50 is used over a window size of 1,000 samples, the baselines scarcely achieve a score of 90\% for the threshold-independent metrics (AUROC, AUPR-In, and AUPR-Out), while other methods achieve near-perfect scores. Moreover, even in the low sample regime (e.g., window size of 10), these baselines fail to demonstrate any superiority over population-based methods.

Finally, based on our experiments, ResNet50 appears to be the best model for achieving low Detection Error and TNR@95TPR; and a large window size is essential for achieving these results. In particular, our method and KS-Softmax both produced impressive outcomes, with Detection Error and TNR@95TPR scores under 0.5 and 1, respectively, when a window size of 1,000 samples was employed. Moreover, we observe that when small window sizes are desired,
ViT-T is the best architecture choice. In particular, our method applied with ViT-T, achieves a phenomenal 86\%+ score in all threshold independent metrics (i.e., AUROC, AUPR-In, AUPR-Out), over a window size of 10 samples, strongly out-preforming the contenders (around 20\% margin). Appendix~\ref{app: Extended Empirical Results} presents the detection performance on a representative selection of distribution shifts, analyzed individually.

\begin{table}[h]
\vspace{-0.2cm}
\resizebox{\textwidth}{!}{%
\begin{tabular}{l||cc||ccccccc}
\multirow{2}{*}{\textbf{Architecture}} &
  \multicolumn{2}{c||}{\multirow{2}{*}{\textbf{Method}}} &
  \multicolumn{7}{c}{\begin{tabular}[c]{@{}c@{}}\textbf{Window size}\\ AUROC~$ \uparrow $~~/~AUPR-In~$ \uparrow $~~/~AUPR-Out~$ \uparrow $~~/~DetectionError~$ \downarrow $~~/~TNR@95TPR~$ \downarrow $~~ \end{tabular}} \\
 &
  \multicolumn{2}{c||}{} &
  10 &
  20 &
  50 &
  100 &
  200 &
  500 &
  1000  \\ \midrule \midrule \
\multirow{7}{*}{ResNet50} & 
\multirow{2}{*}{KS} &
Softmax &  $61/67/62/34/67$  &  $73/74/74/31/64$  &  $87/90/85/13/27$  &   $89/89/89/15/29$  &  $94/95/92/7/14$   &  $\textbf{99}/\textbf{99}/\textbf{99}/\textbf{2}/\textbf{4}^*$   &   $\textbf{100}/\textbf{100}/\textbf{100}/\textbf{0.4}/\textbf{0.9}$   \\ 
 &
   &
  Embeddings  &  $72/\textbf{74}/73/\textbf{28}^*/\textbf{56}^*$  &  $68/73/74/24/48$  &  $81/84/79/18/37$  &  $75/76/79/22/44$   &  $76/79/79/20/40$   &  $84/87/84/13/26$   &   $86/88/84/13/26$   \\ \cmidrule{2-10}
 &
  \multirow{2}{*}{MMD} &
  Softmax  &  $54/61/56/36/72$  &  $62/65/62/37/72$  &  $73/76/72/29/56$  &  $73/73/78/33/59$   &   $79/79/79/35/54$  &   $83/85/83/15/30$  &  $85/85/85/22/37$    \\
 &
   &
  Embeddings &  $75/72/77/38/70$  &  $79/78/79/29/57$  &  $87/87/86/18/37$  &  $83/86/81/15/30$   &  $83/85/82/14/29$   &  $83/85/83/17/32$   &   $83/86/82/13/26$   \\ \cmidrule{2-10} 
 &
  \multirow{2}{*}{Single-instance} &
  SR     &  $56/65/55/34/68$  &  $72/73/72/32/63$  &  $71/75/72/28/56$  &   $77/78/79/25/50$  &  $84/85/83/19/40$   &  $87/88/87/14/28$   &   $88/88/89/15/30$   \\ 
 &
   &
  Entropy  &  $64/69/63/32/64$  &  $73/73/73/32/63$  &  $74/78/73/26/52$  &  $80/80/81/23/47$   &  $84/85/84/17/35$   &  $87/87/87/15/31$   &   $90/90/91/13/26$   \\ \cmidrule{2-10} 
 &
  \multicolumn{2}{c||}{Ours}    &  $\textbf{78}/70/\textbf{82}^*/42/84$  &   $\textbf{88}^*/\textbf{91}^*/\textbf{87}^*/\textbf{15}^*/\textbf{30}^*$ & $\textbf{95}^*/\textbf{95}^*/\textbf{93}^*/\textbf{9}^*/\textbf{17}^*$   &   $\textbf{93}^*/\textbf{93}^*/\textbf{92}^*/\textbf{10}^*/\textbf{20}^*$  &  $\textbf{97}^*/\textbf{97}^*/\textbf{97}^*/\textbf{5}/\textbf{10}$   & $98/98/98/4/7$    &   $\textbf{100}/\textbf{100}/\textbf{100}/\textbf{0.4}/\textbf{0.9}$  \\ \midrule \midrule

\multirow{7}{*}{MobileNetV3-S} &
  \multirow{2}{*}{KS} &
  Softmax  & $71/72/75/\textbf{32}^*/\textbf{63}^*$ & $\textbf{84}^*/84/\textbf{83}/21/43$ & $89/91/88/13/27$ & $92/93/91/\textbf{10}/\textbf{20}$ & $\textbf{95}/\textbf{97}/94/\textbf{5}/\textbf{11}$ & $96/96/\textbf{97}/6/11$ & $\textbf{100}/\textbf{100}/\textbf{100}/\textbf{1}/\textbf{2}$ \\
 &
   &
  Embeddings  & $63/67/63/37/75$ & $65/66/67/37/75$ & $77/78/76/27/54$ & $72/73/76/27/53$ & $84/83/86/22/43$ & $86/87/86/15/30$ & $79/81/81/18/36$  \\ \cmidrule{2-10} 
 &
  \multirow{2}{*}{MMD} &
  Softmax & $75/\textbf{73}/75/38/72$ & $78/78/78/30/59$ & $86/89/82/17/32$ & $86/89/84/14/26$ & $87/88/86/14/28$ & $89/90/88/12/24$ & $90/91/88/11/22$  \\
 &
   &
  Embeddings  & $67/67/68/39/75$ & $66/67/68/37/74$ & $72/77/71/23/47$ & $75/75/79/28/53$ & $89/87/87/20/39$ & $81/82/81/21/40$ & $82/86/80/15/30$ \\ \cmidrule{2-10} 
 &
  \multirow{2}{*}{Single-instance} &
  SR  & $58/62/60/37/74$ & $65/70/66/30/60$ & $86/87/86/19/39$ & $86/88/84/15/30$ & $93/93/93/11/22$ & $96/97/95/5/10$ & $98/98/97/3/7$  \\ 
 &
   &
  Entropy  & $52/61/57/36/72$ & $64/70/65/29/57$ & $85/86/86/17/33$ & $87/88/85/15/29$ & $93/93/94/11/22$ & $96/96/96/6/12$ &  $98/99/98/3/7$ \\\cmidrule{2-10} 
 &
  \multicolumn{2}{c||}{Ours}     & $\textbf{80}^*/\textbf{73}/\textbf{82}^*/40/80$ & $80/\textbf{85}/76/\textbf{20}/\textbf{40}$ & $\textbf{94}^*/\textbf{95}^*/\textbf{93}^*/\textbf{8}^*/\textbf{15}^*$ & $\textbf{94}/\textbf{94}/\textbf{94}^*/11/21$ & $\textbf{95}/96/\textbf{95}/6/13$ & $\textbf{97}/\textbf{98}/96/\textbf{4}/\textbf{8}$ & $99/99/99/\textbf{1}/\textbf{2}$ \\ \midrule \midrule
\multirow{7}{*}{ViT-T} &
  \multirow{2}{*}{KS} &
  Softmax  & $62/68/66/30/59$ & $85/86/83/19/41$ & $82/82/83/21/42$ & $90/91/91/11/22$ & $88/89/90/11/22$ & $95/96/95/5/11$ & $98/98/98/3/6$ \\
 &
   &
  Embeddings &  $68/66/71/38/76$ & $76/83/73/19/\textbf{38}$ & $82/86/80/17/34$ & $81/82/81/20/39$ & $81/84/80/17/34$ & $76/75/82/22/44$ & $84/83/86/19/38$  \\ \cmidrule{2-10} 
 &
  \multirow{2}{*}{MMD} &
  Softmax  & $58/59/64/44/82$ & $69/70/73/38/69$ & $77/80/77/22/44$ & $75/80/76/20/40$ & $80/86/78/15/29$ & $89/91/90/12/23$ & $93/94/92/7/15$  \\
 &
   &
  Embeddings & $61/59/68/46/85$ & $74/77/74/26/53$ & $82/84/80/20/40$ & $80/82/83/21/39$ & $82/81/82/20/40$ & $78/76/81/22/44$ & $77/78/79/24/45$ \\ \cmidrule{2-10} 
 &
  \multirow{2}{*}{Single-instance} &
  SR & $67/70/70/31/61$ & $78/76/77/29/58$ & $76/79/76/23/45$ & $89/90/86/14/29$ & $91/93/91/9/20$ & $97/97/97/5/11$ &  $\textbf{99}/\textbf{99}/98/\textbf{2}/\textbf{5}$ \\ 
 &
   &
  Entropy  & $69/74/69/\textbf{27}/\textbf{53}$ & $79/78/78/27/55$ & $75/78/74/22/44$ & $89/89/90/15/31$ & $86/87/87/14/28$ & $93/94/93/7/14$ & $97/97/97/4/9$ \\ \cmidrule{2-10} 
 &
  \multicolumn{2}{c||}{Ours}        & $\textbf{89}^*/\textbf{86}^*/\textbf{90}^*/28/56$ & $\textbf{91}^*/\textbf{87}/\textbf{92}^*/\textbf{24}/47$ & $\textbf{94}^*/\textbf{93}^*/\textbf{95}^*/\textbf{13}^*/\textbf{25}^*$ & $\textbf{95}^*/\textbf{96}^*/\textbf{96}^*/\textbf{8}^*/\textbf{15}^*$ & $\textbf{97}^*/\textbf{96}^*/\textbf{97}^*/\textbf{8}/\textbf{16}$ & $\textbf{98}^*/\textbf{98}/\textbf{98}/\textbf{4}/\textbf{9}$ & $\textbf{99}/\textbf{99}/\textbf{99}/3/6$ \\ \midrule \midrule
\end{tabular}%
}
\caption{Comparison of different evaluation metrics over \textbf{ResNet50}, \textbf{MobileNetV3-S}, \textbf{ViT-T} with the discussed baselines methods. The best performing method is highlighted in \textbf{bold}; we add the superscript $^*$ to the bolded result when it is statistically significant.}
\label{tab: main results}
\vspace{-0.3cm}
\end{table}

\vspace{-0.3cm}

\section{Concluding Remarks}
\label{sec: Concluding Remarks}
\vspace{-0.2cm}
We presented a novel and powerful method for the detection of distribution shifts within a given window of samples. This coverage-based detection algorithm is theoretically motivated and can be applied to any pretrained model. Due to its low computational complexity, our method, unlike typical baselines, which are order-of-magnitude slower, is practicable. Our comprehensive empirical studies demonstrate that the proposed method works very well, and overall significantly outperforms the baselines on the ImageNet dataset, across a number of neural architectures and a variety of distribution shifts, including adversarial examples. In addition, our coverage bound is of independent interest and allows for the creation of selective classifiers with guaranteed coverage.

Several directions for future research are left open. Although we only considered classification, our method can be extended to regression using an appropriate confidence-rate function such as the MC-dropout \cite{DBLP:conf/icml/GalG16}. Extensions to other tasks, such as object detection and segmentation, would be very interesting. In our method, the information from the multiple coverage bounds was aggregated by averaging, but it is plausible that other statistics or weighted averages could provide more effective detections. Finally, an interesting open question is whether one can benefit from using single-instance or outlier/adversarial detection techniques combined with population-based detection techniques (as discussed here).

\nocite{*}
\bibliography{main}

\begin{thebibliography}{10}

\bibitem{liang2018enhancing}
Liang S, Li Y, Srikant R.
\newblock Enhancing The Reliability of Out-of-distribution Image Detection in
  Neural Networks.
\newblock In: ICLR; 2018. .

\bibitem{DBLP:conf/iclr/HendrycksG17}
Hendrycks D, Gimpel K.
\newblock A Baseline for Detecting Misclassified and Out-of-Distribution
  Examples in Neural Networks.
\newblock In: ICLR; 2017. .

\bibitem{hendrycks2018deep}
Hendrycks D, Mazeika M, Dietterich T.
\newblock Deep Anomaly Detection with Outlier Exposure.
\newblock In: ICLR; 2019. .

\bibitem{DBLP:conf/nips/GolanE18}
Golan I, El-Yaniv R.
\newblock Deep Anomaly Detection Using Geometric Transformations.
\newblock In: NeurIPS; 2018. .

\bibitem{ren2019likelihood}
Ren J, Liu PJ, Fertig E, Snoek J, Poplin R, Depristo M, et~al.
\newblock Likelihood ratios for out-of-distribution detection.
\newblock Advances in Neural Information Processing Systems. 2019.

\bibitem{nalisnick2019detecting}
Nalisnick E, Matsukawa A, Teh YW, Lakshminarayanan B.
\newblock Detecting out-of-distribution inputs to deep generative models using
  typicality.
\newblock arXiv preprint arXiv:190602994. 2019.

\bibitem{DBLP:journals/corr/abs-2106-04015}
Nado Z, Band N, Collier M, Djolonga J, Dusenberry MW, Farquhar S, et~al.
\newblock Uncertainty Baselines: Benchmarks for Uncertainty \& Robustness in
  Deep Learning.
\newblock CoRR. 2021;abs/2106.04015.
\newblock Available from: \url{https://arxiv.org/abs/2106.04015}.

\bibitem{fort2021exploring}
Fort S, Ren J, Lakshminarayanan B.
\newblock Exploring the limits of out-of-distribution detection.
\newblock Advances in Neural Information Processing Systems. 2021.

\bibitem{DBLP:conf/icml/LiptonWS18}
Lipton ZC, Wang YX, Smola AJ.
\newblock Detecting and Correcting for Label Shift with Black Box Predictors.
\newblock In: ICML; 2018. .

\bibitem{Rabanser2019FailingLA}
Rabanser S, G{\"u}nnemann S, Lipton ZC.
\newblock Failing Loudly: An Empirical Study of Methods for Detecting Dataset
  Shift.
\newblock In: NeurIPS; 2019. .

\bibitem{JMLR:v11:el-yaniv10a}
El-Yaniv R, Wiener Y.
\newblock On the Foundations of Noise-free Selective Classification.
\newblock jmlr. 2010.

\bibitem{DBLP:conf/nips/GeifmanE17}
Geifman Y, El-Yaniv R.
\newblock Selective Classification for Deep Neural Networks.
\newblock In: NIPS; 2017. .

\bibitem{wold1987principal}
Wold S, Esbensen K, Geladi P.
\newblock Principal component analysis.
\newblock Chemometrics and intelligent laboratory systems. 1987;2(1-3):37-52.

\bibitem{bingham2001random}
Bingham E, Mannila H.
\newblock Random projection in dimensionality reduction: applications to image
  and text data.
\newblock In: Proceedings of the seventh ACM SIGKDD international conference on
  Knowledge discovery and data mining; 2001. p. 245-50.

\bibitem{rumelhart1985learning}
Rumelhart DE, Hinton GE, Williams RJ.
\newblock Learning internal representations by error propagation.
\newblock California Univ San Diego La Jolla Inst for Cognitive Science; 1985.

\bibitem{DBLP:conf/nips/PuGHYLSC16}
Pu Y, Gan Z, Henao R, Yuan X, Li C, Stevens A, et~al.
\newblock Variational Autoencoder for Deep Learning of Images, Labels and
  Captions.
\newblock In: NIPS; 2016. .

\bibitem{massey1951kolmogorov}
Massey~Jr FJ.
\newblock The Kolmogorov-Smirnov test for goodness of fit.
\newblock Journal of the American statistical Association. 1951.

\bibitem{bland1995multiple}
Bland JM, Altman DG.
\newblock Multiple significance tests: the Bonferroni method.
\newblock Bmj. 1995.

\bibitem{heard2018choosing}
Heard NA, Rubin-Delanchy P.
\newblock Choosing between methods of combining-values.
\newblock Biometrika. 2018;105(1):239-46.

\bibitem{loughin2004systematic}
Loughin TM.
\newblock A systematic comparison of methods for combining p-values from
  independent tests.
\newblock Computational statistics \& data analysis. 2004;47(3):467-85.

\bibitem{DBLP:journals/jmlr/GrettonBRSS12}
Gretton A, Borgwardt KM, Rasch MJ, Schölkopf B, Smola AJ.
\newblock A Kernel Two-Sample Test.
\newblock Journal of Machine Learning Research. 2012;13:723-73.

\bibitem{serfling2009approximation}
Serfling RJ.
\newblock Approximation theorems of mathematical statistics. vol. 162.
\newblock John Wiley \& Sons; 2009.

\bibitem{DBLP:conf/iclr/SutherlandTSDRS17}
Sutherland DJ, Tung HY, Strathmann H, De S, Ramdas A, Smola AJ, et~al.
\newblock Generative Models and Model Criticism via Optimized Maximum Mean
  Discrepancy.
\newblock In: ICLR (Poster); 2017. .

\bibitem{sastry2020detecting}
Sastry CS, Oore S.
\newblock Detecting out-of-distribution examples with gram matrices.
\newblock In: International Conference on Machine Learning. PMLR; 2020. p.
  8491-501.

\bibitem{lee2018simple}
Lee K, Lee K, Lee H, Shin J.
\newblock A simple unified framework for detecting out-of-distribution samples
  and adversarial attacks.
\newblock Advances in neural information processing systems. 2018;31.

\bibitem{cordella1995method}
Cordella LP, De~Stefano C, Tortorella F, Vento M.
\newblock A method for improving classification reliability of multilayer
  perceptrons.
\newblock IEEE Transactions on Neural Networks. 1995;6(5):1140-7.

\bibitem{de2000reject}
De~Stefano C, Sansone C, Vento M.
\newblock To reject or not to reject: that is the question-an answer in case of
  neural classifiers.
\newblock IEEE Transactions on Systems, Man, and Cybernetics, Part C
  (Applications and Reviews). 2000;30(1):84-94.

\bibitem{el2010foundations}
El-Yaniv R, et~al.
\newblock On the Foundations of Noise-free Selective Classification.
\newblock JMLR. 2010.

\bibitem{geifman2018biasreduced}
Geifman Y, Uziel G, El-Yaniv R.
\newblock Bias-Reduced Uncertainty Estimation for Deep Neural Classifiers.
\newblock In: ICLR; 2019. .

\bibitem{langford2005tutorial}
Langford J, Schapire R.
\newblock Tutorial on practical prediction theory for classification.
\newblock Journal of machine learning research. 2005;6(3).

\bibitem{james2013introduction}
James G, Witten D, Hastie T, Tibshirani R.
\newblock An introduction to statistical learning. vol. 112.
\newblock Springer; 2013.

\bibitem{-NMeth2020SciPy}
Virtanen P, Gommers R, Oliphant TE, Haberland M, Reddy T, Cournapeau D, et~al.
\newblock {{SciPy} 1.0: Fundamental Algorithms for Scientific Computing in
  Python}.
\newblock Nature Methods. 2020.

\bibitem{DBLP:conf/cvpr/DengDSLL009}
Deng J, Dong W, Socher R, Li LJ, Li K, Li FF.
\newblock ImageNet: A large-scale hierarchical image database.
\newblock In: CVPR; 2009. .

\bibitem{DBLP:conf/cvpr/HeZRS16}
He K, Zhang X, Ren S, Sun J.
\newblock Deep Residual Learning for Image Recognition.
\newblock In: CVPR; 2016. .

\bibitem{howard2019searching}
Howard A, Sandler M, Chu G, Chen LC, Chen B, Tan M, et~al.
\newblock Searching for mobilenetv3.
\newblock In: Proceedings of the IEEE/CVF international conference on computer
  vision; 2019. p. 1314-24.

\bibitem{dosovitskiy2020image}
Dosovitskiy A, Beyer L, Kolesnikov A, Weissenborn D, Zhai X, Unterthiner T,
  et~al.
\newblock An image is worth 16x16 words: Transformers for image recognition at
  scale.
\newblock arXiv preprint arXiv:201011929. 2020.

\bibitem{rw2019timm}
Wightman R. PyTorch Image Models. GitHub; 2019.
\newblock \url{https://github.com/rwightman/pytorch-image-models}.

\bibitem{43405}
Goodfellow I, Shlens J, Szegedy C.
\newblock Explaining and Harnessing Adversarial Examples.
\newblock In: ICLR; 2015. .

\bibitem{madry2018towards}
Madry A, Makelov A, Schmidt L, Tsipras D, Vladu A.
\newblock Towards Deep Learning Models Resistant to Adversarial Attacks.
\newblock In: ICLR; 2018. .

\bibitem{hendrycks2021nae}
Hendrycks D, Zhao K, Basart S, Steinhardt J, Song D.
\newblock Natural Adversarial Examples.
\newblock CVPR. 2021.

\bibitem{davis2006relationship}
Davis J, Goadrich M.
\newblock The relationship between Precision-Recall and ROC curves.
\newblock In: Proceedings of the 23rd international conference on Machine
  learning; 2006. p. 233-40.

\bibitem{fawcett2006introduction}
Fawcett T.
\newblock An introduction to ROC analysis.
\newblock Pattern recognition letters. 2006;27(8):861-74.

\bibitem{manning1999foundations}
Manning C, Schutze H.
\newblock Foundations of statistical natural language processing.
\newblock MIT press; 1999.

\bibitem{saito2015precision}
Saito T, Rehmsmeier M.
\newblock The precision-recall plot is more informative than the ROC plot when
  evaluating binary classifiers on imbalanced datasets.
\newblock PloS one. 2015;10(3):e0118432.

\bibitem{DBLP:conf/icml/GalG16}
Gal Y, Ghahramani Z.
\newblock Dropout as a Bayesian Approximation: Representing Model Uncertainty
  in Deep Learning.
\newblock In: ICML; 2016. .

\bibitem{DBLP:conf/icml/DavisG06}
Davis J, Goadrich M.
\newblock The relationship between Precision-Recall and ROC curves.
\newblock In: ICML; 2006. .

\bibitem{cover1967nearest}
Cover T, Hart P.
\newblock Nearest neighbor pattern classification.
\newblock IEEE transactions on information theory. 1967;13(1):21-7.

\bibitem{sriperumbudur2010hilbert}
Sriperumbudur BK, Gretton A, Fukumizu K, Sch{\"o}lkopf B, Lanckriet GR.
\newblock Hilbert space embeddings and metrics on probability measures.
\newblock The Journal of Machine Learning Research. 2010;11:1517-61.

\bibitem{krizhevsky2009learning}
Krizhevsky A, Hinton G, et~al.
\newblock Learning multiple layers of features from tiny images. 2009.

\bibitem{DBLP:conf/sp/Carlini017}
Carlini N, Wagner DA.
\newblock Towards Evaluating the Robustness of Neural Networks.
\newblock In: IEEE Symposium on Security and Privacy; 2017. .

\bibitem{37648}
Netzer Y, Wang T, Coates A, Bissacco A, Wu B, Ng AY.
\newblock Reading Digits in Natural Images with Unsupervised Feature Learning.
\newblock In: NIPS Workshop; 2011. .

\bibitem{hendrycks2018benchmarking}
Hendrycks D, Dietterich T.
\newblock Benchmarking Neural Network Robustness to Common Corruptions and
  Perturbations.
\newblock In: ICLR; 2019. .

\bibitem{DBLP:conf/icml/TanL19}
Tan M, Le QV.
\newblock EfficientNet: Rethinking Model Scaling for Convolutional Neural
  Networks.
\newblock In: ICML; 2019. .

\bibitem{zhai2021scaling}
Zhai X, Kolesnikov A, Houlsby N, Beyer L.
\newblock Scaling vision transformers.
\newblock arXiv preprint arXiv:210604560. 2021.

\bibitem{gal2016uncertainty}
Gal Y.
\newblock Uncertainty in deep learning. 2016.

\bibitem{galil2023can}
Galil I, Dabbah M, El-Yaniv R.
\newblock What can we learn from the selective prediction and uncertainty
  estimation performance of 523 imagenet classifiers.
\newblock arXiv preprint arXiv:230211874. 2023.

\bibitem{5206848}
Deng J, Dong W, Socher R, Li LJ, Li K, Fei-Fei L.
\newblock ImageNet: A large-scale hierarchical image database.
\newblock In: 2009 IEEE Conference on Computer Vision and Pattern Recognition;
  2009. .

\bibitem{guillory2021predicting}
Guillory D, Shankar V, Ebrahimi S, Darrell T, Schmidt L.
\newblock Predicting with confidence on unseen distributions.
\newblock In: ICCV; 2021. p. 1134-44.

\bibitem{kossen2021active}
Kossen J, Farquhar S, Gal Y, Rainforth T.
\newblock Active testing: Sample-efficient model evaluation.
\newblock In: ICML. PMLR; 2021. p. 5753-63.

\bibitem{Bengio+chapter2007}
Bengio Y, LeCun Y.
\newblock Scaling Learning Algorithms Towards {AI}.
\newblock In: Large Scale Kernel Machines. MIT Press; 2007. .

\bibitem{Hinton06}
Hinton GE, Osindero S, Teh YW.
\newblock A Fast Learning Algorithm for Deep Belief Nets.
\newblock Neural Computation. 2006;18:1527-54.

\bibitem{goodfellow2016deep}
Goodfellow I, Bengio Y, Courville A, Bengio Y.
\newblock Deep learning. vol.~1.
\newblock MIT Press; 2016.

\end{thebibliography}
\bibliographystyle{vancouver}

\newpage
\appendix

\section{Shift-Detection General Framework}
\label{app:Shift Detection General Framework}
The general framework for shift-detection can be found in the following figure, Figure~\ref{typical pipeline}.

\begin{figure}[h]
\begin{center}
\large
\includegraphics[width=\columnwidth]{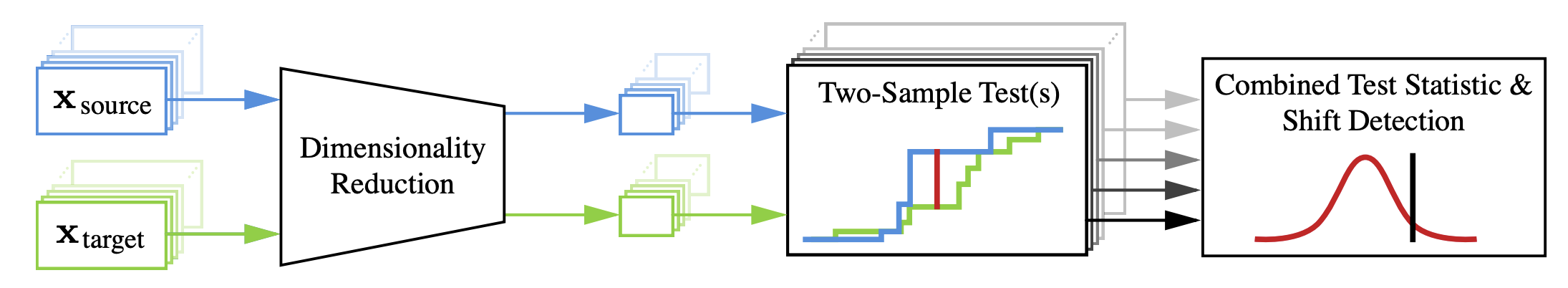}
\caption{The procedure of detecting a dataset shift using dimensionality reduction and then a two-sample statistical test. The dimensionality reduction is applied to both the detection-training (source) and test (target) data, prior to being analyzed using statistical hypothesis testing. This figure is taken from \cite{Rabanser2019FailingLA}.}
\label{typical pipeline}
\end{center}
\end{figure}

\section{Proofs}
\label{Proofs}

\subsection{Proof for Theorem~\ref{Theorem:uniform convergence}}
\label{Proof:Theorem:uniform convergence}
\begin{proof}
Define
\begin{eqnarray}
    && {\cal B}_{\theta_i} \triangleq b^*_i(m, m \cdot \hat{c}_i(\theta_i, S_m), \frac{\delta}{k}) \nonumber, \\
    && {\cal C}_{\theta_i} \triangleq c(\theta_i,P) \nonumber.
\end{eqnarray}

Consider the $\text{i}^{\text{th}}$ iteration of SGR over a detection-training set $S_m$, and recall that, $\theta_i = \kappa_f(x_z)$, $x_z \in S_m$ (see Algorithm~\ref{Alg:CD}). Therefore, $\theta_i$ is a random variable (between zero and one), since it is a function of a random variable ($x \in S_m$).
Let $\guyPr_{S_m} \{ \theta_i = \theta' \}$ be the probability that  $\theta_i =  \theta'$.

Therefore,
\begin{eqnarray}
&& \guyPr_{S_m} \{{\cal C}_{\theta_i} < {\cal B}_{\theta_i} \} \nonumber \\
&& =  \int_{0}^{1} \,d\theta'  \guyPr_{S_m} \{{\cal C}_{\theta_i} < {\cal B}_{\theta_i}| \theta_i = \theta' \} \cdot \guyPr_{S_m} \{ \theta_i = \theta' \} \nonumber \\
&&=  \int_{0}^{1} \,d\theta'  \guyPr_{S_m} \{{\cal C}_{\theta'} < {\cal B}_{\theta'}\} \cdot \guyPr_{S_m} \{ \theta_i = \theta' \} \nonumber.
\label{law of total probability}
\end{eqnarray}
Since ${\cal B}_{\theta_i}$ is obtained using Lemma~\ref{lemma} (see Algorithm~\ref{Alg:CD}), and $\theta_i = \theta'$,

\begin{equation}
    \guyPr_{S_m} \{{\cal C}_{\theta_i} < {\cal B}_{\theta_i} \} = \guyPr_{S_m} \{{\cal C}_{\theta'} < {\cal B}_{\theta'} \} < \frac{\delta}{k}, \nonumber
\end{equation}
so we get,
\begin{eqnarray}
&&  \guyPr_{S_m} \{{\cal C}_{\theta_i} < {\cal B}_{\theta_i} \} \nonumber \\
&& = \int_{0}^{1} \,d\theta'  \guyPr_{S_m} \{{\cal C}_{\theta'} < {\cal B}_{\theta'}\} \cdot \guyPr_{S_m} \{ \theta_i = \theta' \} \nonumber \\
&& < \int_{0}^{1} \,d\theta' \frac{\delta}{k} \cdot \guyPr_{S_m} \{ \theta_i = \theta' \} \nonumber \\
&& =  \frac{\delta}{k} \cdot \left( \int_{0}^{1} \,d\theta' \guyPr_{S_m} \{ \theta_i = \theta' \} \right) \nonumber \\
&& = \frac{\delta}{k}.
\label{finishing proof}
\end{eqnarray}
The following application of the union bound completes the proof,

\begin{eqnarray}
\guyPr_{S_m} \{\exists i : {\cal C}_{\theta_i} < {\cal B}_{\theta_i} \}  \leq \sum_{i=1}^k \guyPr_{S_m} \{{\cal C}_{\theta_i} < {\cal B}_{\theta_i} \} < \sum_{i=1}^k \frac{\delta}{k} = \delta. \nonumber 
\label{finishing proof}
\end{eqnarray}

\end{proof}

\newpage

\section{Exploring Model Sensitivity: Evaluating Accuracy on Shifted Datasets}
\label{app: Exploring Model Sensitivity: Evaluating Accuracy on Shifted Datasets}
In this section, we present Table~\ref{tab: Acc}, which displays the accuracy (when applicable) as well as the degradation from the original accuracy over the ImageNet dataset, of the considered models on each of the simulated shifts mentioned in Section~\ref{sec: Distribution Shift Datasets}.

\begin{table}[h]
\resizebox{\textwidth}{!}{%
\begin{tabular}{c||c|c||c|c||c|c|}
\multirow{2}{*}{Shift Dataset} & \multicolumn{2}{c||}{\textbf{ResNet50}} & \multicolumn{2}{c||}{\textbf{MovileNetV3}} & \multicolumn{2}{c|}{\textbf{ViT-T}}      \\ 
                               & Acc.  & ImageNet Degradation & Acc.    & ImageNet Degradation  & Acc.  & ImageNet Degradation \\ \midrule  \midrule  
FGSM $\epsilon = 7\cdot10^{-5}$     & 76.68\% & -3.7\%   & 62.09\% & -3.15\%  & 72.51\% & -2.95\%  \\ \midrule
FGSM $\epsilon = 1\cdot10^{-4}$  & 75.19\% & -5.19\%  & 60.72\% & -4.52\%  & 71.49\% & -3.97\%  \\ \midrule
FGSM $\epsilon = 3\cdot10^{-4}$   & 66.15\% & -14.23\% & 52.09\% & -13.15\% & 65.06\% & -10.4\%  \\ \midrule
FGSM $\epsilon = 5\cdot10^{-4}$   & 59.23\% & -21.15\% & 44.45\% & -20.79\% & 58.9\%  & -16.56\% \\ \midrule
PGD $\epsilon = 1\cdot10^{-4}$ & 74.64\% & -5.74\%  & 60.63\% & -4.61\%  & 71.35\% & -4.11\%  \\ \midrule
GAUSSIAN $\sigma = 0.1$ & 79.02\% & -1.36\%  & 62.82\% & -2.42\%  & 71.79\% & -3.67\%  \\ \midrule
GAUSSIAN $\sigma = 0.3$ & 74.63\% & -5.75\%  & 55.06\% & -10.18\% & 50.86\% & -24.6\%  \\ \midrule
GAUSSIAN $\sigma = 0.5$ & 68.56\% & -11.82\% & 42.55\% & -22.69\% & 22.25\% & -53.21\% \\ \midrule
GAUSSIAN $\sigma = 1$ & 46.1\%  & -34.28\% & 13.82\% & -51.42\% & 0.56\%  & -74.9\%  \\ \midrule
ZOOM $50\%$     & 65.55\% & -14.83\% & 36.96\% & -28.28\% & 46.04\% & -29.42\% \\ \midrule
ZOOM $70\%$  & 74.31\% & -6.07\%  & 53.53\% & -11.71\% & 62.69\% & -12.77\% \\ \midrule
ZOOM $90\%$     & 78.6\%  & -1.78\%  & 61.28\% & -3.96\%  & 72.08\% & -3.38\%  \\ \midrule
ROTATION $\theta = 5^{\circ}$ & 76.7\%  & -3.68\%  & 62.42\% & -2.82\%  & 71.27\% & -4.19\%  \\ \midrule
ROTATION $\theta = 10^{\circ}$ & 72.4\%  & -7.98\%  & 58.22\% & -7.02\%  & 67.29\% & -8.17\%  \\ \midrule
ROTATION $\theta = 20^{\circ}$ & 68.29\% & -12.09\% & 49.96\% & -15.28\% & 62.38\% & -13.08\% \\ \midrule
ROTATION $\theta = 25^{\circ}$ & 70.08\% & -10.3\%  & 50.95\% & -14.29\% & 60.97\% & -14.49\% \\ \midrule \midrule
\end{tabular}%
}
\caption{Shifted dataset accuracy and comparison with ImageNet. We displays the accuracy results for each shifted dataset and model combination, along with the accuracy degradation when compared to the original ImageNet dataset.}
\label{tab: Acc}
\end{table}

\newpage
\section{Extended Empirical Results}
\label{app: Extended Empirical Results}
In this section, we present a detailed analysis of our empirical findings on the ResNet50 architecture. We report the results for each window size, $|W_k| \in \{ 10, 20, 50, 100, 200, 500, 1000 \}$, and for several shift cases discussed in Section~\ref{sec: Distribution Shift Datasets}. In particular, we show the detection performance of all the discussed methods, for the following shifts: FGSM (Table~\ref{tab: ResNet50 - fgsm}), ImageNet-O (Table~\ref{tab: ResNet50 ImageNet-O}), ImageNet-A (Table~\ref{tab: ResNet50 ImageNet-A}), and the Zoom out shift, 90\% (Table~\ref{tab: ResNet50 Zoom}).

\begin{table}[h]
\resizebox{\textwidth}{!}{%
\begin{tabular}{cc||ccccccc}
\multicolumn{2}{c||}{\multirow{2}{*}{\textbf{Method}}} & \multicolumn{7}{c}{\begin{tabular}[c]{@{}c@{}}\textbf{Window size}\\ AUROC~$ \uparrow $~~/~AUPR-In~$ \uparrow $~~/~AUPR-Out~$ \uparrow $~~/~DetectionError~$ \downarrow $~~/~TNR@95TPR~$ \downarrow $~~ \end{tabular}} \\
\multicolumn{2}{c||}{}                          & 10 & 20 & 50 & 100 & 200 & 500 & 1000 \\ \midrule \midrule
\multirow{2}{*}{KS}              & Softmax    &  $32/45/40/47/92$  &  $47/55/46/44/91$  & $64/72/59/34/69$ & $72/72/75/38/77$  &   $80/87/69/18/36$  &  $\textbf{100}/\textbf{100}/\textbf{100}/\textbf{2}/\textbf{4}$  &  $\textbf{100}/\textbf{100}/\textbf{100}/\textbf{0}/\textbf{0}$    \\
                                 & Embeddings  &  $54/58/55/\textbf{43}/\textbf{86}$  &  $32/39/48/49/100$  & $54/64/49/39/80$ &  $41/44/48/50/99$ &  $37/48/45/47/92$   &  $60/70/59/30/60$  &   $71/77/61/33/68$  \\ \midrule
\multirow{2}{*}{MMD}             & Softmax   &  $36/48/42/45/90$  & $44/56/44/42/82$   & $51/53/51/48/93$ & $41/44/54/49/97$  &  $50/52/50/48/94$   &  $48/52/51/45/93$  &   $55/55/55/47/94$    \\
                                 & Embeddings &  $61/56/60/48/95$  & $57/57/59/45/93$   & $72/73/67/36/73$ & $63/70/56/38/71$  &   $63/69/55/37/75$  & $67/70/61/39/74$   &  $70/79/59/28/54$   \\ \midrule
\multirow{2}{*}{Single-instance} & SR       &  $34/45/40/47/93$  &   $69/68/72/42/82$ & $43/52/50/45/90$ &  $54/54/61/47/93$ &  $62/64/58/42/86$   &  $66/73/59/35/72$  &   $72/69/73/43/86$   \\ 
                                 & Entropy  &  $42/49/44/47/94$  &  
 $65/60/65/47/92$ & $49/55/49/45/89$ & $59/53/63/49/98$  &  $60/66/58/39/77$   &  $59/59/56/45/90$  &   $64/61/63/46/90$   \\ \midrule
\multicolumn{2}{c||}{Ours}                 &  $\textbf{71}/\textbf{64}/\textbf{75}/45/92$  &  $\textbf{77}/\textbf{82}/\textbf{75}/\textbf{25}/\textbf{51}^*$  &  $\textbf{88}/\textbf{90}/\textbf{85}/\textbf{20}/\textbf{39}$ &  $\textbf{84}/\textbf{86}/\textbf{84}/\textbf{25}/\textbf{49}$ &  $\textbf{99}^*/\textbf{99}^*/\textbf{99}^*/\textbf{3}^*/\textbf{5}^*$   &  $98/98/98/5/10$  &   $\textbf{100}/\textbf{100}/\textbf{100}/2/2$   \\ \midrule \midrule
\end{tabular}%
}
\caption{Comparison of different evaluation metrics over \textbf{ResNet50} with the discussed baselines methods, over the FGSM shift with $\epsilon= 0.0001$. The best performing method is highlighted in \textbf{bold}; we add the superscript $^*$ to the bolded result when it is statistically significant.}
\label{tab: ResNet50 - fgsm}
\end{table}

\begin{table}[h]
\resizebox{\textwidth}{!}{%
\begin{tabular}{cc||ccccccc}
\multicolumn{2}{c||}{\multirow{2}{*}{\textbf{Method}}} & \multicolumn{7}{c}{\begin{tabular}[c]{@{}c@{}}\textbf{Window size}\\ AUROC~$ \uparrow $~~/~AUPR-In~$ \uparrow $~~/~AUPR-Out~$ \uparrow $~~/~DetectionError~$ \downarrow $~~/~TNR@95TPR~$ \downarrow $~~ \end{tabular}} \\
\multicolumn{2}{c||}{}                          & 10 & 20 & 50 & 100 & 200 & 500 & 1000 \\ \midrule \midrule
\multirow{2}{*}{KS}              & Softmax    & $62/59/60/46/94$ & $70/61/75/48/94$ & $98/98/98/6/11$ & $99/99/99/5/10$  &  $\textbf{100}/\textbf{100}/\textbf{100}/\textbf{0}/\textbf{0}$    &  $\textbf{100}/\textbf{100}/\textbf{100}/\textbf{0}/\textbf{0}$  &   $\textbf{100}/\textbf{100}/\textbf{100}/\textbf{0}/\textbf{0}$   \\
                                 & Embeddings  & $85/89/76/18/35$ & $\textbf{97}/\textbf{98}/\textbf{97}/\textbf{6}^*/\textbf{12}^*$  & $99/99/99/5/9$ &  $\textbf{100}/\textbf{100}/\textbf{100}/\textbf{0}/\textbf{0}$ &   $\textbf{100}/\textbf{100}/\textbf{100}/\textbf{0}/\textbf{0}$  &  $\textbf{100}/\textbf{100}/\textbf{100}/\textbf{0}/\textbf{0}$   &  $\textbf{100}/\textbf{100}/\textbf{100}/\textbf{0}/\textbf{0}$   \\ \midrule
\multirow{2}{*}{MMD}             & Softmax   & $43/53/47/44/87$ & $74/75/73/39/72$ & $94/92/96/33/51$ & $97/97/98/31/27$ &  $97/97/98/32/26$  &  $\textbf{100}/\textbf{100}/\textbf{100}/\textbf{0}/\textbf{0}$  &   $\textbf{100}/\textbf{100}/\textbf{100}/\textbf{0}/\textbf{0}$    \\
                                 & Embeddings & $\textbf{96}/\textbf{97}/\textbf{97}^*/\textbf{10}/\textbf{20}$ & $95/94/97/33/39$ & $\textbf{100}/\textbf{100}/\textbf{100}/\textbf{0}/\textbf{0}$ &  $\textbf{100}/\textbf{100}/\textbf{100}/\textbf{0}/\textbf{0}$ &  $\textbf{100}/\textbf{100}/\textbf{100}/\textbf{0}/\textbf{0}$    &   $94/95/97/31/46$ &   $\textbf{100}/\textbf{100}/\textbf{100}/\textbf{0}/\textbf{0}$  \\ \midrule
\multirow{2}{*}{Single-instance} & SR       & $62/63/60/43/87$ &  $31/41/38/48/98$ & $47/56/44/41/86$ & $56/57/63/45/85$ &  $51/51/53/48/97$   &  $37/41/42/50/100$  &   $42/44/47/49/100$   \\ 
                                 & Entropy  & $64/66/68/40/81$ & $39/42/53/50/99$ & $41/52/45/44/90$ & $58/66/52/37/75$ &  $54/51/55/49/99$  &  $52/55/53/48/90$  &   $84/85/84/26/52$   \\ \midrule
\multicolumn{2}{c||}{Ours}                 &  $61/61/61/47/92$ & $84/89/77/18/37$  & $99/99/99/5/8$ & $\textbf{100}/\textbf{100}/\textbf{100}/\textbf{0}/\textbf{0}$  &  $\textbf{100}/\textbf{100}/\textbf{100}/\textbf{0}/\textbf{0}$    &  $\textbf{100}/\textbf{100}/\textbf{100}/\textbf{0}/\textbf{0}$   &    $\textbf{100}/\textbf{100}/\textbf{100}/\textbf{0}/\textbf{0}$   \\ \midrule \midrule
\end{tabular}%
}

\caption{Comparison of different evaluation metrics over \textbf{ResNet50} with the discussed baselines methods, over the ImageNet-O shift. The best performing method is highlighted in \textbf{bold}; we add the superscript $^*$ to the bolded result when it is statistically significant.}
\label{tab: ResNet50 ImageNet-O}
\end{table}

\begin{table}[h]
\resizebox{\textwidth}{!}{%
\begin{tabular}{cc||ccccccc}
\multicolumn{2}{c||}{\multirow{2}{*}{\textbf{Method}}} & \multicolumn{7}{c}{\begin{tabular}[c]{@{}c@{}}\textbf{Window size}\\ AUROC~$ \uparrow $~~/~AUPR-In~$ \uparrow $~~/~AUPR-Out~$ \uparrow $~~/~DetectionError~$ \downarrow $~~/~TNR@95TPR~$ \downarrow $~~ \end{tabular}} \\
\multicolumn{2}{c||}{}                          & 10 & 20 & 50 & 100 & 200 & 500 & 1000 \\ \midrule \midrule
\multirow{2}{*}{KS}              & Softmax    &   $\textbf{100}/\textbf{100}/\textbf{100}/\textbf{0}/\textbf{0}$  & $\textbf{100}/\textbf{100}/\textbf{100}/\textbf{0}/\textbf{0}$    & $\textbf{100}/\textbf{100}/\textbf{100}/\textbf{0}/\textbf{0}$  &  $\textbf{100}/\textbf{100}/\textbf{100}/\textbf{0}/\textbf{0}$  &   $\textbf{100}/\textbf{100}/\textbf{100}/\textbf{0}/\textbf{0}$   &  $\textbf{100}/\textbf{100}/\textbf{100}/\textbf{0}/\textbf{0}$   &   $\textbf{100}/\textbf{100}/\textbf{100}/\textbf{0}/\textbf{0}$    \\
                                 & Embeddings  &  $98/98/98/7/14$  &  $\textbf{100}/\textbf{100}/\textbf{100}/\textbf{0}/\textbf{0}$   & $\textbf{100}/\textbf{100}/\textbf{100}/\textbf{0}/\textbf{0}$  &  $\textbf{100}/\textbf{100}/\textbf{100}/\textbf{0}/\textbf{0}$  &  $\textbf{100}/\textbf{100}/\textbf{100}/\textbf{0}/\textbf{0}$    &  $\textbf{100}/\textbf{100}/\textbf{100}/\textbf{0}/\textbf{0}$   &  $\textbf{100}/\textbf{100}/\textbf{100}/\textbf{0}/\textbf{0}$    \\ \midrule
\multirow{2}{*}{MMD}             & Softmax   &  $\textbf{100}/\textbf{100}/\textbf{100}/\textbf{0}/\textbf{0}$   &  $\textbf{100}/\textbf{100}/\textbf{100}/\textbf{0}/\textbf{0}$   & $\textbf{100}/\textbf{100}/\textbf{100}/\textbf{0}/\textbf{0}$  & $\textbf{100}/\textbf{100}/\textbf{100}/\textbf{0}/\textbf{0}$   &   $\textbf{100}/\textbf{100}/\textbf{100}/\textbf{0}/\textbf{0}$   &   $\textbf{100}/\textbf{100}/\textbf{100}/\textbf{0}/\textbf{0}$  &  $\textbf{100}/\textbf{100}/\textbf{100}/\textbf{0}/\textbf{0}$      \\
                                 & Embeddings &  $\textbf{100}/\textbf{100}/\textbf{100}/\textbf{0}/\textbf{0}$   & $\textbf{100}/\textbf{100}/\textbf{100}/\textbf{0}/\textbf{0}$    & $\textbf{100}/\textbf{100}/\textbf{100}/\textbf{0}/\textbf{0}$  &  $\textbf{100}/\textbf{100}/\textbf{100}/\textbf{0}/\textbf{0}$  &   $\textbf{100}/\textbf{100}/\textbf{100}/\textbf{0}/\textbf{0}$   &  $\textbf{100}/\textbf{100}/\textbf{100}/\textbf{0}/\textbf{0}$   &  $\textbf{100}/\textbf{100}/\textbf{100}/\textbf{0}/\textbf{0}$    \\ \midrule
\multirow{2}{*}{Single-instance} & SR     &  $\textbf{100}/\textbf{100}/\textbf{100}/\textbf{0}/\textbf{0}$   & $\textbf{100}/\textbf{100}/\textbf{100}/\textbf{0}/\textbf{0}$    & $\textbf{100}/\textbf{100}/\textbf{100}/\textbf{0}/\textbf{0}$  &  $\textbf{100}/\textbf{100}/\textbf{100}/\textbf{0}/\textbf{0}$  &   $\textbf{100}/\textbf{100}/\textbf{100}/\textbf{0}/\textbf{0}$   &  $\textbf{100}/\textbf{100}/\textbf{100}/\textbf{0}/\textbf{0}$   &  $\textbf{100}/\textbf{100}/\textbf{100}/\textbf{0}/\textbf{0}$      \\ 
                                 & Entropy &  $\textbf{100}/\textbf{100}/\textbf{100}/\textbf{0}/\textbf{0}$   & $\textbf{100}/\textbf{100}/\textbf{100}/\textbf{0}/\textbf{0}$    & $\textbf{100}/\textbf{100}/\textbf{100}/\textbf{0}/\textbf{0}$  &  $\textbf{100}/\textbf{100}/\textbf{100}/\textbf{0}/\textbf{0}$  &   $\textbf{100}/\textbf{100}/\textbf{100}/\textbf{0}/\textbf{0}$   &  $\textbf{100}/\textbf{100}/\textbf{100}/\textbf{0}/\textbf{0}$   &  $\textbf{100}/\textbf{100}/\textbf{100}/\textbf{0}/\textbf{0}$     \\ \midrule
\multicolumn{2}{c||}{Ours}                 &   $\textbf{100}/\textbf{100}/\textbf{100}/\textbf{0}/\textbf{0}$   & $\textbf{100}/\textbf{100}/\textbf{100}/\textbf{0}/\textbf{0}$    & $\textbf{100}/\textbf{100}/\textbf{100}/\textbf{0}/\textbf{0}$  &  $\textbf{100}/\textbf{100}/\textbf{100}/\textbf{0}/\textbf{0}$  &   $\textbf{100}/\textbf{100}/\textbf{100}/\textbf{0}/\textbf{0}$   &  $\textbf{100}/\textbf{100}/\textbf{100}/\textbf{0}/\textbf{0}$   &  $\textbf{100}/\textbf{100}/\textbf{100}/\textbf{0}/\textbf{0}$    \\ \midrule \midrule
\end{tabular}%
}
\caption{Comparison of different evaluation metrics over \textbf{ResNet50} with the discussed baselines methods, over the ImageNet-A shift. The best performing method is highlighted in \textbf{bold}; we add the superscript $^*$ to the bolded result when it is statistically significant.}
\label{tab: ResNet50 ImageNet-A}
\end{table}

\begin{table}[h]
\resizebox{\textwidth}{!}{%
\begin{tabular}{cc||ccccccc}
\multicolumn{2}{c||}{\multirow{2}{*}{\textbf{Method}}} & \multicolumn{7}{c}{\begin{tabular}[c]{@{}c@{}}\textbf{Window size}\\ AUROC~$ \uparrow $~~/~AUPR-In~$ \uparrow $~~/~AUPR-Out~$ \uparrow $~~/~DetectionError~$ \downarrow $~~/~TNR@95TPR~$ \downarrow $~~ \end{tabular}} \\
\multicolumn{2}{c||}{}                          & 10 & 20 & 50 & 100 & 200 & 500 & 1000 \\ \midrule \midrule
\multirow{2}{*}{KS}              & Softmax    &  $42/49/48/47/93$  &  $53/52/54/47/97$  & $61/69/57/37/74$ &  $71/72/69/39/77$ &  $\textbf{91}/\textbf{92}/\textbf{91}^*/\textbf{15}/\textbf{30}$   &  $\textbf{100}^*/\textbf{100}^*/\textbf{100}^*/\textbf{2}^*/\textbf{3}^*$  &   $\textbf{100}/\textbf{100}/\textbf{100}/\textbf{0}/\textbf{0}$    \\
                                 & Embeddings  &  $46/51/52/46/92$  &  $25/45/38/44/87$  & $56/58/54/44/86$ & $46/46/53/49/98$  &   $35/44/42/49/97$  & $27/38/38/50/99$   &   $41/46/46/47/97$  \\ \midrule
\multirow{2}{*}{MMD}             & Softmax   &  $48/51/50/\textbf{46}/93$  &  $46/44/52/50/100$  & $51/57/52/45/88$ & $50/56/52/46/88$  &  $52/57/49/46/88$   &  $51/57/47/44/90$  &  $53/53/54/49/96$     \\
                                 & Embeddings &  $57/56/61/48/\textbf{91}$  &  $54/63/49/41/83$  & $68/67/73/40/82$ & $37/53/41/43/83$  &  $28/41/38/49/97$  &  $31/38/41/50/100$  &   $18/35/34/49/100$  \\ \midrule
\multirow{2}{*}{Single-instance} & SR       &  $28/38/38/50/100$  &  $45/45/48/51/99$  & $44/53/53/43/87$ & $52/59/53/42/84$  &   $65/68/59/41/85$  &  $74/73/78/38/78$  &  $84/84/83/31/62$    \\ 
                                 & Entropy  &  $29/38/39/50/100$  &  $50/47/56/51/100$  & $51/53/58/46/91$ & $61/63/66/41/83$  &  $71/74/61/36/72$   &  $76/73/81/39/77$  &    $87/86/87/29/58$  \\ \midrule
\multicolumn{2}{c||}{Ours}                 &  $\textbf{70}/\textbf{61}/\textbf{77}/47/95$  &    $\textbf{69}/\textbf{73}/\textbf{68}/\textbf{35}/\textbf{70}$ & $\textbf{81}/\textbf{86}/\textbf{75}/\textbf{22}/\textbf{43}$ & $\textbf{72}/\textbf{75}/\textbf{71}/\textbf{34}/\textbf{67}$  &   $76/82/71/23/46$  &  $94/94/94/15/31$  &   $\textbf{100}/\textbf{100}/\textbf{100}/\textbf{0}/\textbf{0}$    \\ \midrule \midrule
\end{tabular}%
}
\caption{Comparison of different evaluation metrics over \textbf{ResNet50} with the discussed baselines methods, over the Zoom out (90\%) shift. The best performing method is highlighted in \textbf{bold}; we add the superscript $^*$ to the bolded result when it is statistically significant.}
\label{tab: ResNet50 Zoom}
\end{table}

\newpage
\section{Ablation Study}
\label{app: Ablation Study}
In this section, we conduct multiple experiments to analyze the various components of our framework; all those experiments are conducted using a ResNet50. We explore several hyper-parameter choices, including $C_{\text{target}}$, $\delta$, and $\kappa_f$. More specifically, we consider $C_{\text{target}} \in \{ 1, 10, 100 \} $, and $\delta \in \{ 0.1, 0.01, 0.001, 0.0001 \}$, and two different CFs $\kappa_f$, namely SR and Entropy-based.

To evaluate the performance of our detectors under varying hyper-parameters, we have selected a single metric that we believe to be the most important, namely, AUROC~\cite{galil2023can}. Additionally, since performance may vary depending on window size, we display the average AUROC across all window sizes that we have considered in our experiments. These window sizes include: $\{ 10 ,20, 50, 100, 200, 500, 1000\}$. In Figure~\ref{fig: heatmaps}, we summarize our findings by displaying the average AUROC value as a function of the chosen hyper-parameters. These results are presented as heatmaps.

\begin{figure}[H]
\begin{subfigure}[H]{0.46\linewidth}
\includegraphics[width=1.1\linewidth]{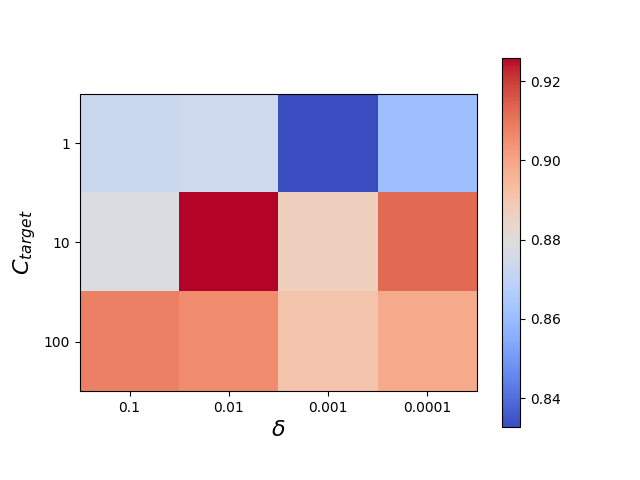}
\caption{$\kappa_f= (1- \text{ Entropy})$.}
\label{fig: heatmaps Ent}
\end{subfigure}
\begin{subfigure}[H]{0.46\linewidth}
\includegraphics[width=1.1\linewidth]{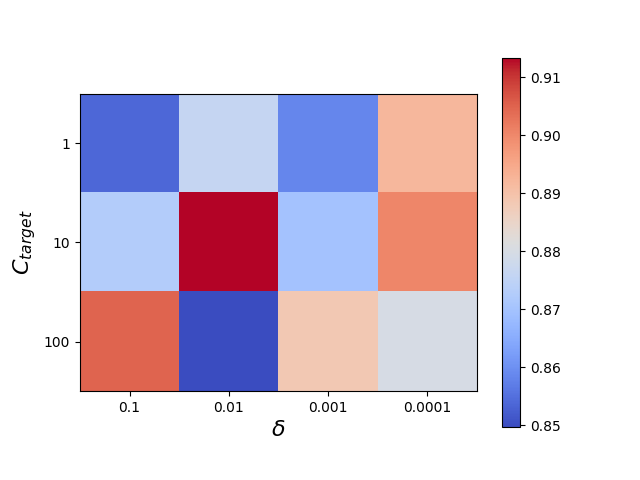}
\caption{$\kappa_f= \text{SR}$.}
\label{fig: heatmaps SR}
\end{subfigure}%
\caption{AUROC performance of our detector under different choices of hyper-parameters.}
\label{fig: heatmaps}
\end{figure}

Figure~\ref{fig: heatmaps Ent}, displays our AUROC detector's performance when we use Entropy-based as our CF. We observe that the optimal choice of hyper-parameters is $\delta = 0.01$ and $C_{\text{target}}=10$, resulting in the highest performance. However, increasing the value of $C_{\text{target}}$ leads to a more consistent and robust detector, as changes in the value of $\delta$ do not significantly affect the detector's performance. Additionally, we note that using $C_{\text{target}}=1$ yields relatively poor performance, indicating that a single coverage choice is insufficient to capture the characteristics of the distribution represented by the sample $S_m$. Similar results are obtained when using SR as the CF, as shown in Figure~\ref{fig: heatmaps SR}. These results suggest that selecting a high value of $C_{\text{target}}$ and a low value of $\delta$ is the most effective approach for ensuring a robust detector. Finally, the heatmaps demonstrate that Entropy-based CF outperforms (by a low margin) SR, in terms of detection performance.

\end{document}